\documentclass[a4paper,fleqn]{cas-dc}

\usepackage[numbers]{natbib}
\usepackage{amsthm}
\usepackage{pifont}

\usepackage{hyperref}

\usepackage{algorithm}
\usepackage{algorithmic}

\usepackage[capitalize,noabbrev]{cleveref}

\theoremstyle{plain}
\newtheorem{theorem}{Theorem}[section]
\newtheorem{proposition}[theorem]{Proposition}

\theoremstyle{definition}

\theoremstyle{remark}

\newcommand{\cmark}{\ding{51}}
\newcommand{\xmark}{\ding{55}}

\def\tsc#1{\csdef{#1}{\textsc{\lowercase{#1}}\xspace}}
\tsc{WGM}
\tsc{QE}
\tsc{EP}
\tsc{PMS}
\tsc{BEC}
\tsc{DE}

\begin{document}
\let\WriteBookmarks\relax
\def\floatpagepagefraction{1}
\def\textpagefraction{.001}
\shorttitle{SC3D for Temporal and Instantaneous Causal Discovery}
\shortauthors{S. Das et~al.}

\title [mode = title]{SC3D: Dynamic and Differentiable Causal Discovery for Temporal and Instantaneous Graphs}                      



\author[1]{Sourajit Das}[
                        orcid=0000-0001-7307-7453]
\cormark[1]
\ead{szd399@psu.edu}

\credit{Conceptualization of this study, Methodology, Software, Writing}

\affiliation[1]{organization={Institute for Computational Data Science},
                addressline={Pennsylvania State University}, 
                city={State College},
                postcode={16801}, 
                state={Pennsylvania},
                country={USA}}


\author[1]{Dibyajyoti Chakraborthy}[%
   ]
\ead{d.chakraborty@psu.edu}

\credit{Data curation, Software, Methodology}


\author[3]{Romit Maulik}
\ead{rmaulik@purdue.edu}

\affiliation[3]{organization={School of Mechanical Engineering},
                addressline={Purdue University}, 
                city={West Lafayette},
                postcode={47906}, 
                state={Indiana}, 
                country={USA}}

\cortext[cor1]{Corresponding author}

\credit{Project supervision, conceptualization, reviewing}


\begin{abstract}
Discovering causal structures from multivariate time series is a key problem because interactions span across multiple lags and possibly involve instantaneous dependencies. Additionally, the search space of the dynamic graphs is combinatorial in nature. In this study, we propose \textit{Stable Causal Dynamic Differentiable Discovery (SC3D)}, a two-stage differentiable framework that jointly learns lag-specific adjacency matrices and, if present, an instantaneous directed acyclic graph (DAG). In Stage 1, SC3D performs edge preselection through node-wise prediction to obtain masks for lagged and instantaneous edges, whereas Stage 2 refines these masks by optimizing a likelihood with sparsity along with enforcing acyclicity on the instantaneous block. Numerical results across synthetic SVAR systems, nonlinear and chaotic benchmarks, nonstationary dynamics and real-world datasets demonstrate that SC3D achieves improved stability and more accurate recovery of both lagged and instantaneous causal structures compared to existing  baselines.
\end{abstract}




\maketitle

\section{Introduction}
\label{sec:intro}
The understanding of cause-effect relationships in a dynamical systems can be pivotal for scientific discovery and reliable decision making. Many application fields such as climate and geophysical systems, neuroscience, biological systems and engineered networks, generate measurable data in the form of multivariate time series. In such situations, causal mechanisms can propagate over time as well instantaneously. In general predictive models can be accurate without being causal whereas causal structure learning seeks to retrieve directed graphs which facilitate mechanistic intervention and under favorable assumptions, allows for reasoning under interventions and distribution shifts \cite{pearl2009causality, peters2017elements}. In this work, we propose \textit{Stable Causal Dynamic Differentiable Discovery (SC3D)}, a two-stage differentiable framework for causal discovery in multivariate dynamical systems that jointly learns both lag specific causal relationships and an instantaneous directed acyclic graph (DAG). Here \textit{dynamic} refers to causal dependencies across time lags rather than time varying causal graphs.

A conventional approach for causal discovery in time series data is based on Granger-style prediction \cite{granger1969investigating} which measures directed dependencies by improving forecasting accuracy. This perspective inspired a large section of the research community to focus on lagged causal learning in nonlinear and vector autoregressive models. This includes constraint based methods and conditional independence testing curated for high dimensional data \cite{runge2019detecting}. Other methods leverage functional and distributional assumptions like non-Gaussianity, to detect directed structures in temporal networks such as VAR-LiNGAM \cite{hyvarinen2010lingam}. In recent years, deep learning has also been exploited to parameterize nonlinear predictors and derive causal scores from learned temporal relationships \cite{tank2021neural, nauta2019causal}. In spite of the above progress, two challenges continue to be pressing for structural time series models: (i) identifying individual components of a high-dimensional multivariate time-series signal which causally influence one another, including both lagged inter-slice and instantaneous intra-slice effects under structural constraints, (ii) scaling the structure to larger number of variables without combinatorial search.

At the same time, differentiable causal discovery (DCD) has emerged as a powerful framework for learning directed acyclic graphs (DAGs) from observational data by formulating the structure learning as an optimization problem over a set of weighted adjacency matrices \cite{zheng2018dags}. However, generic smooth acyclicity constraints can be numerically unstable and deteriorate very quickly beyond tens of variables \cite{nazaret2023stable}. Stable Differentiable Causal Discovery (SDCD) \cite{nazaret2023stable} addresses this challenge through a numerically stable spectral acyclicity constraint and a two stage optimization technique, significantly improve scalability and robustness. For the case of time series, differentiable approaches like DYNOTEARS \cite{pamfil2020dynotears} exploit continuous optimization for discovering the dynamic Bayesian networks structure in a temporal and intra-slice manner. On the other hand, newer methods include NTS-NOTEARS \cite{sun2023ntsnotears}, which expands on the differentiable temporal structure learning to incorporate nonlinearity, while TECDI \cite{li2023causal} examines the interventional discovery of causality over time. However, extension of robust and scalable methods in the differentiable domain to structural equation model (SEM) is still not fully explored.

In this paper, we propose a stable differentiable framework for causal discovery in multivariate dynamical systems with lagged and instantaneous effects. We study a nonlinear dynamic structural equation model (SEM) of order $L$, where we consider a special case of the linear structural vector autoregressive model (VAR). Here the lagged edges represent the temporal propagation and the instantaneous edges represent intra-slice causal relationships. Motivated from the  stability aspects of the SDCD model by \cite{nazaret2023stable}, we present a two-step approach for dynamic causal discovery. Stage 1 conducts a node wise predictive screening over a time window to find candidate masks for lag-specific and instantaneous edges. Stage 2 refined these candidate masks through optimization of a likelihood based objective function, along with enforcing the spectral acyclicity penalty on the instantaneous block. This design ensures that the instantaneous block is acyclic while retaining the strengths of differentiable optimization while avoiding the problems of earlier smooth constraints \cite{zheng2018dags, nazaret2023stable}. We summarize our contributions to the paper below:
\begin{itemize}
    \item We formulate causal discovery for dynamical systems using lag-specific adjacency matrices and instantaneous DAG components, bridging structural VAR models with stable differentiable causal discovery.
    \item We introduce a two stage screening and constrained refinement approach for temporal causal discovery, by screening candidate parents within a lagged predictor window prior to enforcing acyclicity constraint on the instantaneous matrix.
    \item We provide population level guarantees proving that Stage 1 preserves all true causal parents within the chosen temporal window, and that the acyclicity of the instantaneous dependencies implies acyclicity of the time unrolled graph.
    \item We evaluate \textsc{SC3D} on synthetic, benchmark and real-world temporal datasets, demonstrating improved causal structure recovery over existing baselines.
\end{itemize}
\vspace{-0.5cm}

\begin{table}[t]
\caption{Capabilities of representative baselines for temporal (lagged) and instantaneous causal discovery. ``Instantaneous DAG'' indicates an explicit acyclic intra-slice graph; ``Nonlinear'' indicates support for nonlinear dependencies.}
\label{tab:capabilities}
\centering
\scriptsize
\setlength{\tabcolsep}{4pt}
\begin{tabular}{lcccc}
\toprule
Method          & Lagged  & Instantaneous  & Nonlinear & Differentiable \\
\midrule
\textsc{SC3D}      & \cmark & \cmark & \cmark & \cmark \\
DYNOTEARS          & \cmark & \xmark & \xmark & \cmark \\
PCMCI+             & \cmark & \xmark & \cmark & \xmark \\
NTS-NOTEARS        & \cmark & \cmark & \cmark & \cmark \\
TECDI              & \cmark & \cmark & \cmark & \cmark \\
VAR-LiNGAM         & \cmark & \cmark & \xmark & \xmark \\
NeuralGC (cMLP)    & \cmark & \xmark & \cmark & \xmark \\
TCDF               & \cmark & \xmark & \cmark & \xmark \\
\bottomrule
\end{tabular}
\vspace{-0.5cm}
\end{table}

\section{Background and Related Work}
\label{sec:background}

Fundamental research on causal discovery from temporal data traces back to Granger causality, where causality is measured by improvement in forecasting by upon addition of past history of a variable to the model ~\cite{granger1969investigating, shojaie2022granger}.  This principle underlines a wide range of vector autoregressive (VAR) and nonlinear autoregressive models, which have been applied in econometrics, neuroscience, climate science and complex analysis ~\cite{kleinberg2013causality, smith2011network, zhang2011climate}. Conventional approaches consist of constraint-based methods that leverage conditional independence testing such as PCMCI and PCMCI+ ~\cite{runge2019detecting, runge2020discovering}. These methods iteratively test if a candidate driver stays statistically dependent of the target variable after conditioning on selected past and instantaneous variables, pruning edges whose associations can be defined by the rest of the temporal system. These methods scale to moderately large systems although they rely heavily on test assumptions and hyperparameters. There is also the use of functional and distributional methods such VAR-LiNGAM, which exploit non-Gaussianity to recover directed temporal structures but require strict assumptions and deteriorate at higher dimensions \cite{hyvarinen2010lingam}.

In the recent years, researchers have explored neural parameterization of Granger causality using recurrent or convolutional neural network (RNN or CNN) architectures \cite{tank2021neural, khanna2019economy, marcinkevics2021interpretable, nauta2019causal}. Despite such networks being flexible enough and capable of capturing nonlinear temporal relationships, these approaches predominantly focus on lagged interactions and do not model explicit instantaneous causal relationships. Furthermore, RNNs are often sensitive to regularization and thresholding, rendering unstable causal discovery in sparse or high dimensional spaces. Amortized versions and representation learning approaches have relaxed some of the modeling assumptions but the result do not directly enforce graph level structural adjacency supports, or validity of the estimated scores as a well defined causal graph~\cite{lowe2022amortized, li2023crvae}.

Differentiable causal discovery has emerged to be a strong candidate for learning DAGs, by formulating the structure learning as a continuous optimization problem \cite{zheng2018dags}. Recent advancements have attempted to further speed up the optimization efficiency  and theoretical properties using approaches such as topological swaps and alternative penalties and neural network parameterizations \cite{deng2023topological, yu2019daggnn, yu2021nocurl}. However, most of these methods have utilized smooth trace-based acyclicity constraints, which can become numerically unstable or restrictive as the dimensions become large \cite{nazaret2023stable}. So Stable Differentiable Causal Discovery (SDCD) solved this bottleneck by introducing a new kind of spectral acyclicity constraint along with a two-stage optimization strategy, significantly improving the scalability for static DAGs \cite{nazaret2023stable}.

If we extend differentiable causal discovery to time series, we are faced with additional challenges, because such dynamic temporal systems display both lagged (inter-slice) and instantaneous (intra-slice) relationships. DYNOTEARS \cite{pamfil2020dynotears} extends continuous time optimization to dynamic Bayesian networks and estimates contemporaneous and time-lagged linear relationships. NTS-NOTEARS also extends the temporal structure learning approach of NOTEARS to the nonparametric dynamic Bayesian network setting using neural function approximation, allowing nonlinear lagged and instantaneous dependencies \cite{sun2023ntsnotears}. TECDI focuses on causal discovery of temporally structured graphs in the interventional setting by employing differentiable acyclicity constraints \cite{li2023causal}. Similarly, UnCLe~\cite{bi2025uncle} performs scalable causal discovery in nonlinear temporal systems with a focus on lagged dependencies, but does not model instantaneous causal graphs or enforce acyclicity constraints, making it complementary rather than directly comparable to the joint SVAR setting considered here. 

In general, existing methods either model \emph{lagged and instantaneous dependencies independently} or suffer from \emph{scalability and stability issues} when trying to learn both simultaneously. Unlike the existing methods, \textsc{SC3D} explicitly models both lag-specific temporal adjacency matrices as well as an instantaneous directed acyclic graph. Compared to the baseline methods which focus only on modeling the temporal dynamics using a purely linear temporal differentiable methods, the novel approach here leverages nonlinear dynamic structural equation models with the linear SVAR model as a special case. \textsc{SC3D} is designed to operate on observational multivariate time series data. The key methodological distinction is the combination of node-wise temporal preselection, masked likelihood based refinement along with constrained optimization. The optimization leverages a spectral and 2-cycle penalty applied only on the instantaneous causal structure to recover both lagged and instantaneous causal structure in a scalable and stable way. This novel approach avoids numerical instabilities of smooth acyclicity constraints while minimizing the combinatorial search space and tackles joint dynamic causal discovery. The implementation and code repository of \textsc{SC3D} is available at:
\url{https://github.com/ISCLPennState/SC3D}. 

\begin{figure}
    \centering
    \includegraphics[width=1.1\columnwidth]{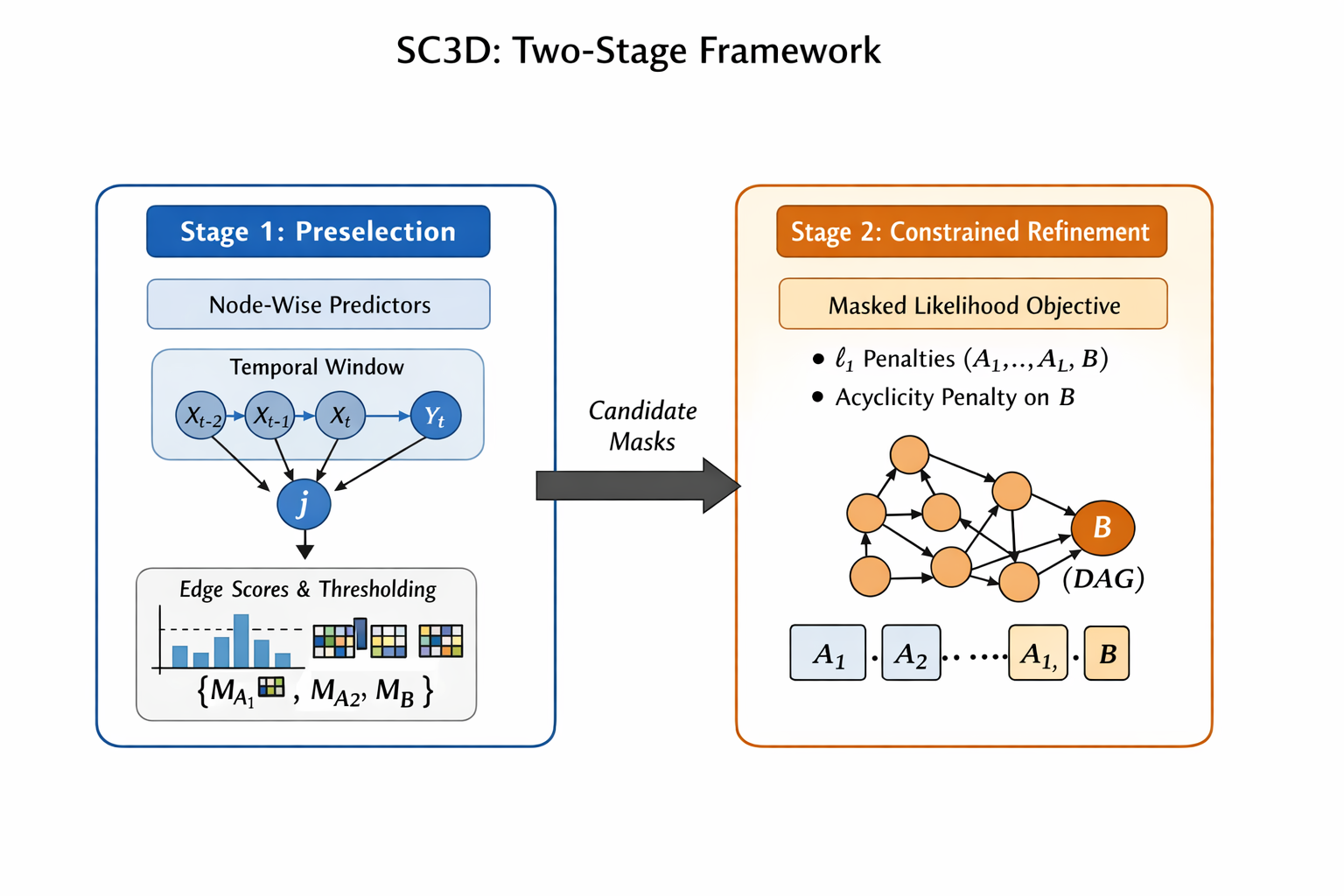}
    \vspace{-0.5cm}
    \caption{Overview of the proposed \textsc{SC3D} framework. Stage~1 performs node-wise preselection over a temporal window to generate candidate masks for lagged and instantaneous edges. Stage~2 refines the masked structure by optimizing a regularized likelihood objective, while enforcing acyclicity on the instantaneous graph $B$.}
    \label{fig:sc3d_overview}
\end{figure}

\section{Methodology}
\label{sec:method}

\begin{figure}
  \centering
  \includegraphics[width=\columnwidth]{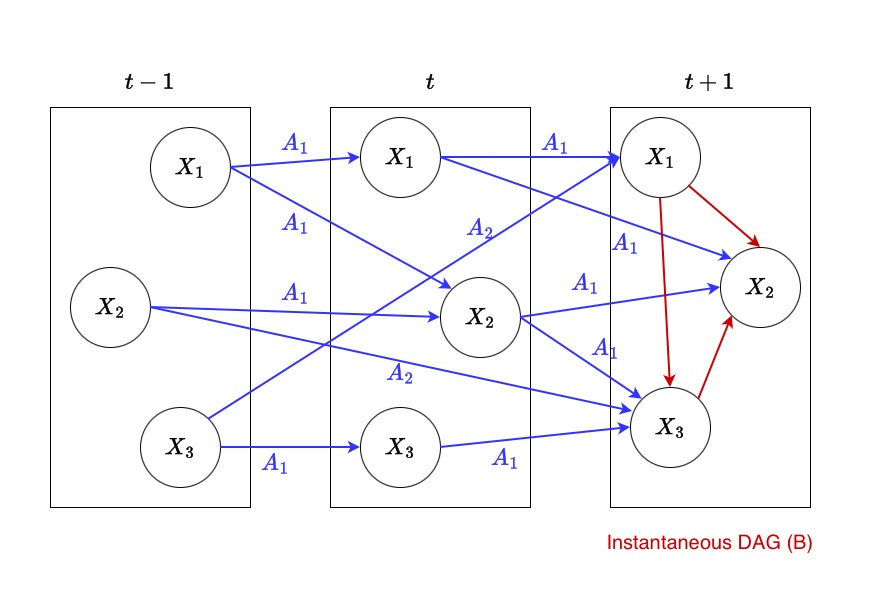}
  \caption{
  Structural vector autoregressive (SVAR) model with lagged and instantaneous causal dependencies. Nodes represent variables across time slices. Blue arrows denote lagged causal effects encoded by matrices $\{A_\ell\}_{\ell=1}^L$,
  while red arrows denote instantaneous causal relationships encoded by a directed acyclic graph $B$. \textsc{SC3D} jointly estimates both components, enforcing acyclicity only on the instantaneous block.
  }
  \label{fig:svar_model}
  \vskip -0.2in
\end{figure}

\subsection{Problem Formulation}
\label{sec:problem_formulation}

We consider a $d$-dimensional multivariate time series
$\{X_t\}_{t\in\mathbb{Z}}$ which is generated by a nonlinear dynamic structural equation model of order $L$. For any subset $S \subseteq [d]$, let
$X_t^S := (X_t^i)_{i \in S}$ represent the subvector indexed by $S$. Now for each node $j \in [d]$, the nonlinear dynamic SEM can be expressed as:
\begin{multline}
X_{t+1}^j = f_j^\star\!\left(X_{t+1}^{\mathrm{pa}_0(j)},
X_t^{\mathrm{pa}_1(j)}, \ldots, X_{t+1-L}^{\mathrm{pa}_L(j)},
\varepsilon_{t+1}^j \right), \\
j=1,\ldots,d,
\label{eq:nonlinear_dsem}
\end{multline}
where $\mathrm{pa}_0(j)$ denotes the contemporaneous parent set of node $j$,
$\mathrm{pa}_\ell(j)$ denotes its corresponding lag-$\ell$ parent set, and $\varepsilon_{t+1}^1,\ldots,\varepsilon_{t+1}^d$ represent the jointly independent noise variables independent of $\{X_s\}_{s\le t}$. We assume the contemporaneous graph induced by $\{\mathrm{pa}_0(j)\}_{j=1}^d$ to be acyclic.

We denote the causal graphs using lag-specific adjacency supports $\{A_\ell^\star\}_{\ell=1}^L$ and the instantaneous adjacency support by
$B^\star$. Throughout the entire article, we follow the convention that the row index
represents the target variable while the column index represents the source variable:
\begin{equation}
\label{eq:support_convention}
\begin{aligned}
(A_\ell^\star)_{ji}=1 &\quad \Longleftrightarrow \quad
i\in \mathrm{pa}_\ell(j), \\
(B^\star)_{ji}=1&\quad \Longleftrightarrow \quad
i\in \mathrm{pa}_0(j).
\end{aligned}
\end{equation}
Thus, $(A_\ell^\star)_{ji}=1$ denotes the lagged edge $X_{t+1-\ell}^i\rightarrow X_{t+1}^j$, while $(B^\star)_{ji}=1$ denotes the instantaneous edge
$X_{t+1}^i\rightarrow X_{t+1}^j$. Given that we have observations $\{X_t\}_{t=1}^T$, the objective of causal discovery is to recover the lagged adjacency supports $\{A_\ell^\star\}_{\ell=1}^L$ and the instantaneous adjacency support $B^\star$, rather than the exact function $\{f_j^\star\}_{j=1}^d$.

The linear structural vector autoregressive model (SVAR) is a special case of \eqref{eq:nonlinear_dsem}. Particularly, if the structural functions are linear and the noise is additive, we can write:
\begin{equation}
X_{t+1} = B^\star X_{t+1} + \sum_{\ell=1}^L A_\ell^\star X_{t+1-\ell} + \varepsilon_{t+1},
\label{eq:linear_svar}
\end{equation}
where $A_\ell^\star\in\mathbb{R}^{d\times d}$ denotes the
lag-$\ell$ coefficient matrix and $B^\star\in\mathbb{R}^{d\times d}$ denotes the instantaneous coefficient matrix. We consider the assumption that $(I-B^\star)$ is invertible and the graph induced by $B^\star$ is acyclic in nature. Figure~\ref{fig:svar_model} depicts the SVAR structure with lagged and instantaneous dependencies used in our study. Since the lagged edges are ordered by time, they cannot form directed cycles in the time-unrolled graph. In contrast, instantaneous edges occur within the same time slice and are required to satisfy an acyclicity constraint to represent a well-defined structural model.

\subsection{Overview of \textsc{SC3D}}
\label{sec:sc3d_overview}

The causal discovery problem in~\eqref{eq:nonlinear_dsem} is challenging because we have to account for a combinatorial search space due to the joint lagged and instantaneous dependencies, as well as numerical instability when enforcing acyclicity constraints in high dimensional spaces. In order to address these challenges, we propose \textsc{SC3D}, a two-stage differentiable framework for causal discovery in multivariate dynamical systems comprising of both lagged and instantaneous dependencies. The proposed method takes a multivariate time series as input and gives the outputs: (i) lag-specific adjacency matrices $\{A_\ell\}_{\ell=1}^L$ encoding temporal inter-slice causal relationships and (ii) an instantaneous directed acyclic graph (DAG), $B$, representing intra-slice causal relationships. The overall design is motivated from the stability and scalability aspects of differentiable causal discovery, while explicitly utilizing the temporal ordering intrinsic to time-series data.

In Stage 1, we perform temporal screening for each target variable (node) by fitting conditional predictive models using a fixed time window of candidate parents. This step identifies a sparse set of candidate lagged and instantaneous edges by removing weak dependencies while keeping all true causal parents at the population level. The output of Stage 1 is a pair of binary masks that limit the admissible lagged and instantaneous edges in the next stage. In Stage 2, we refine the causal structure by re-optimizing the model parameters under the masks from Stage 1 while enforcing acyclicity on the instantaneous adjacency matrix. We gradually increase the penalty coefficient until the instantaneous graph becomes acyclic, following which we keep it fixed. This design utilizes the time ordering of the lagged edges to ensure numerical stability and controlled optimization for the instantaneous structure.

\subsection{Stage 1: Node-wise Temporal Preselection}
\label{sec:stage1}

Stage 1 aims to narrow down the search space by removing unlikely causal edges early while retaining the true dynamic parents at the population level. This is accomplished by node-wise temporal screening, which discovers possible lagged and instantaneous parents for all variables without any acyclicity constraint. For each target variable $X_{t+1}^j$, let us consider a temporal predictor window 
\[
\mathcal{V}_t^{(j)} := \{ X_{t+1}^{-j}, X_t, X_{t-1}, \dots, X_{t+1-L} \}.
\]
which includes the lagged variables up to order $L$ and, if instantaneous effects are allowed, contemporaneous variables without self-loops. This window defines the pool of candidate parents for $X_{t+1}^j$. 

Now we define the \emph{dynamic Markov boundary} within the window $\mathcal{V}_t^{(j)}$, denoted by $\mathrm{MB}^{\mathrm{dyn}}(j)$ is the minimal set
$S \subseteq \mathcal{V}_t^{(j)}$ such that
\[
X_{t+1}^j \;\perp\!\!\!\perp\;
\mathcal{V}_t^{(j)} \setminus S \;\big|\; S.
\]
Intuitively, $\mathrm{MB}^{\mathrm{dyn}}(j)$ includes all lagged and contemporaneous variable within the window that directly affect the data generating process of $X_{t+1}^j$.

At the population level, this stage is equivalent to maximizing a penalized conditional log-likelihood over subsets of predictors. For any subset $S \subseteq \mathcal{V}_t^{(j)}$, we define the Stage-1 (population) score
\begin{equation}
\Psi_j(S) := \sup_{\theta:\,\mathrm{supp}(\theta)\subseteq S}\;
\mathbb{E}\big[\log p_\theta(X_{t+1}^j \mid X_S)\big] \;-\;\lambda |S|.
\label{eq:stage1-pop}
\end{equation}
where $p_\theta$ represents a node-wise conditional prediction model
parameterized using a differentiable predictor. In our experiments, the
predictor is implemented as a multilayer perceptron (MLP), which maps
the candidate temporal window $\mathcal{V}_t^{(j)}$ to the target variable $X_{t+1}^j$. For our application purposes, we use a Gaussian conditional likelihood, where the mean is represented by the output of the MLP. With a fixed observation variance, the maximizing the conditional log-likelihood is the same as the minimizing the squared error prediction loss. The parameter $\lambda>0$ controls the sparsity of the selected input groups. \cref{thm:stage1} highlights the role of Stage 1 and proves that the preselection step does not discard any true dynamic parents within the temporal window. 

\begin{theorem}
    \label{thm:stage1}
    Using the assumptions enlisted in Appendix~\ref{app:stage1-assumptions}, there exists $\lambda_0 > 0$ such that for all $0 < \lambda \leq \lambda_0$, any maximizer
    \begin{equation}
        S_j^\star \in \arg\max_{S \subseteq \mathcal{V}_t^{(j)}} \Psi_j(S)
    \end{equation}
    contains the dynamic Markov boundary for $X_{t+1}^j$. In particular,
    \begin{equation}
        \mathrm{MB}^{\mathrm{dyn}}(j) \subseteq S_j^\star,
    \end{equation}
    i.e. Stage 1 does not discard any true dynamic parents in the window. 
\end{theorem}

Here, $\Psi_j(S)$ denotes the Stage-1 population objective for node $j$, defined as the penalized conditional log-likelihood over predictors $S$. The assumptions of \cref{thm:stage1} and its proof are derived in Appendix \ref{app:stage1-assumptions}. The proof is inspired by population-level selection arguments used in SDCD for static DAGs \cite{nazaret2023stable}, and is adapted to fit the temporal dynamic SEM setting in our case.
 
In reality, we fit node-wise conditional models that predict $X_{t+1}^j$ from subsets of $\mathcal{V}_t^{(j)}$, parameterized using differentiable predictors. The edge weights are then extracted from the grouped input layer weight magnitudes, producing nonnegative scores for each lagged and instantaneous parent. The sparsity is enforced through $\ell_1$ regularization on these group magnitudes. The learned group magnitudes are thresholded to create binary masks for lagged and instantaneous edges which are then used in Stage 2.

\paragraph{Stage-1 Optimization}
In practice, Stage~1 can be implemented by optimizing the empirical version of the
population score in~\eqref{eq:stage1-pop}.
Specifically, for each node $j$, we aim to solve
\begin{equation}
\hat{\theta}_j = \arg\max_{\theta} \frac{1}{n}\sum_{t=1}^n
\log p_\theta(X_{t+1}^j \mid X_{\mathcal{V}_t^{(j)}}) - \lambda \|\theta\|_1,
\label{eq:stage1_empirical}
\end{equation}
where the $\ell_1$ penalty enhances sparsity in the learned input groups.
The resulting group magnitudes are then thresholded to generate the binary masks to be
used in Stage~2.

\subsection{Stage 2: Constrained Structure Refinement with Spectral Acyclicity}
\label{sec:stage2}
Stage 2 learns the causal structure by refining the model parameters based on the mask which were obtained in Stage 1, but this time with the intention of enforcing acyclicity on the instantaneous causal graph. All the disallowed edges from Stage 1 are now set to zero and the optimization is limited to the remaining lagged and instantaneous candidates. We now optimize regularized likelihood objective function that encourages sparsity in both the lagged and instantaneous adjacency matrices. This constrained re-optimization significantly reduces the search space and stabilizes training in high-dimensional settings. 

Let $\{A_\ell\}_{\ell=1}^L$ denote the set of lag-specific adjacency matrices and $B$ denote the instantaneous adjacency matrix. Since lagged edges move forward in time by design, they do not need additional constraints.  Thus, we enforce an explicit constraint only on the instantaneous block $B$.

\paragraph{Stage-2 Optimization Problem}
Let $\{M_{A_\ell}\}_{\ell=1}^L$ and $M_B$ represent the binary masks obtained from
Stage~1 screening for lagged and instantaneous edges, respectively. The
learned matrices can be restricted to the admissible entries by imposing
\begin{equation}
\label{eq:mask_constraints}
\begin{aligned}
A_\ell &= A_\ell \odot M_{A_\ell}, \qquad \ell=1,\ldots,L, \\
B &= B \odot M_B,
\end{aligned}
\end{equation}
where $\odot$ denotes element-wise multiplication of the matrices. Using the binary masks we obtained in Stage 1, we restrict both the lagged and instantaneous edges to the corresponding masked entries. The goal is to now estimate the matrices $\{A_\ell\}_{\ell=1}^L$ and $B$ by minimizing a regularized negative log-likelihood function, while enforcing acyclicity on the instantaneous block $B$. For the evaluation of the negative log-likelihood, we use the same MLP-based Gaussian conditional models as above, but limited to the admissible inputs based on Stage~1 masks. Concretely, the Stage-2 optimization problem can be written as follows:
\begin{equation}
\begin{aligned}
\min_{\{A_\ell\}_{\ell=1}^L,\, B} \quad & \mathcal{L}_{\mathrm{nll}}(\{A_\ell\}, B)
+ \alpha \sum_{\ell=1}^L \|A_\ell\|_1 \\
& + \beta \|B\|_1 + \lambda_{\mathrm{2c}} \| B \odot B^\top \|_{1} \\
\text{s.t.} \quad & \rho(|B|) = 0,
\end{aligned}
\label{eq:stage2_constrained}
\end{equation}
where $\mathcal{L}_{\mathrm{nll}}$ is the negative log-likelihood of the node-wise conditional models, $\|\cdot\|_1$ is the $\ell_1$ norm which promotes sparsity, and $\rho(|B|)$ denotes the spectral radius of absolute instantaneous adjacency matrix$B$.


In practice, we solve a relaxation of the acyclicity constraint in~\eqref{eq:stage2_constrained}: 
\begin{multline}
    \label{eq:stage2_obj}
    \min_{\{A_\ell\}_{\ell=1}^L,\, B}\;\;
    \mathcal{L}_{\mathrm{nll}}(\{A_\ell\}, B) \;+\; 
    \alpha \sum_{\ell=1}^L \| A_\ell \|_{1} \;+\; \beta \| B \|_{1} \\ \;+\;
    \gamma \, \rho(|B|) \;+\; \lambda_{\mathrm{2c}} \, \| B \odot B^\top \|_{1}. 
\end{multline}
under the mask constraints defined in \eqref{eq:mask_constraints}. Here $B \odot B^\top$ represents the element wise product of $B$ with its transpose. The spectral radius penalty, $\rho(B)$, equals to zero if and only if the underlying directed graph induced by $B$ is acyclic. 

This approach provides a numerically stable alternative to smooth trace based acyclicity constraints~\cite{nazaret2023stable}. The spectral radius tends to enforce global acyclicity but it does not explicitly prevent short directed cycle during the optimization. When dealing with instantaneous graphs of large dimensions, it is possible to encounter strong 2-cycles composed of pairs of opposing edges that may satisfy a small spectral radius but potentially lead to unstable or ambiguous causal assignments. In order to address, we augment the spectral constraint with a 2-cycle penalty $\|B \odot B^\top\|_1$, which directly penalizes the reciprocal instantaneous edges. This additional term discourages ambiguous reciprocal dependencies and stabilizes optimization in high dimensional regimes without negating the smoothness and scalability properties. 

In order to prevent the over penalization of the acyclicity constraint, we adopt a gradual penalty schedule, $\gamma$. In other words, we begin with $\gamma$ set to zero and increase it linearly during the training phase. At regular intervals, we extract a DAG from the current estimate of $B$ by retaining the highest magnitude edges that preserve acyclicity. Once the extracted instantaneous graph satisfies the acyclicity constraint in terms of the target edge budget, we fix the penalty coefficient $\gamma$ for the rest of the training. This approach impedes the oscillatory behavior and stabilizes optimization upon convergence.

Since the lagged edges are only able to connect nodes from earlier time slices to later time slices, it is sufficient to impose acyclicity on the instantaneous block $B$ to guarantee that the acyclicity of the time-unrolled causal graph. The following proposition formalizes this insight showing that enforcing acyclicity on the instantaneous block alone is sufficient to guarantee acyclicity of the time-unrolled graph. 

\begin{proposition}
    \label{prop:unrolled_acyclic}
    Consider a finite time window $T\in\mathbb{N}$, time unrolled graph generated by the dynamic SEM~\eqref{eq:nonlinear_dsem} over the window $\{X_1,\dots,X_T\}$ with instantaneous edges denoted as $B^\star$ and lag-$\ell$ edges denoted as $A_\ell^\star$. If the directed graph generated by $B^\star$ is acyclic, then the time unrolled graph over the window $\{1, \dots, T\}$ is acyclic for any lagged structures $\{A_\ell^\star\}_{\ell=1}^L$.
\end{proposition}
The proof of \cref{prop:unrolled_acyclic} is given in Appendix~\ref{app:unrolled_acyclic}.

\subsection{Identifiability Scope}
\label{sec:identifiability}

The analysis above address two properties of the proposed framework: Stage~1
screening preserves the true dynamic parent set under population-level assumptions, and acyclicity of the instantaneous graph is sufficient to guarantee acyclicity of the time-unrolled graph. However, complete identifiability of the target causal graph from observational time-series data needs additional assumptions on the data generation process. We therefore state the scope within which the lagged and instantaneous graph supports are identifiable.

\begin{proposition}
\label{prop:ident_dyn}
Using Assumption~\ref{ass:ident_dyn}, the observational distribution of the process uniquely determines the complete time unrolled causal graph within the assumed model class. Hence, both the lagged adjacency supports $\{A_\ell^\star\}_{\ell=1}^L$and the instantaneous adjacency support $B^\star$ are identifiable within the model class.
\end{proposition}

The proof of Proposition~\ref{prop:ident_dyn} is provided in
Appendix~\ref{app:proof_ident_dyn}.

\paragraph{Remark:}
Proposition~\ref{prop:ident_dyn} is an identifiability statement for the
restricted additive-noise dynamic SEM in Assumption~\ref{ass:ident_dyn}. It
makes no claim about the identifiability of arbitrary observational SVAR or nonlinear dynamic SEMs. Specifically, even in the case of linear Gaussian models, identifiability can not be fully guaranteed without additional assumptions, such as non-Gaussian disturbances or equal error variances \cite{peters2014identifiability}.

\subsection{Algorithm}
\label{sec:algorithm}

The first stage of \textsc{SC3D} performs node-wise temporal preselection in order to determine potential lagged and instantaneous edges. In Stage~2, it refines the causal structure via optimization of regularized likelihood objective function based on the resulting masks while ensuring acyclicity on the instantaneous block through a spectral acyclicity penalty. The algorithm in ~\cref{alg:dynamic_sdcd} outlines the proposed two-stage approach for the dynamic causal discovery along with the important hyperparameters considered in the experiments.

Penalty freezing is triggered when the extracted instantaneous graph satisfies
the acyclicity criterion and retains at least
\[
E_{\min}=\lfloor \eta(d)s_{\mathrm{inst}}d \rfloor
\]
edges, where $s_{\mathrm{inst}}$ denotes the expected instantaneous indegree.
In the reported experiments, we set $s_{\mathrm{inst}}=2$. The retain fraction $\eta(d)$
is chosen adaptively based on dimension:
\[
\eta(d) =
\begin{cases}
0.5, & d \le 6, \\
0.65, & 7 \le d \le 20, \\
0.8, & d > 20.
\end{cases}
\]
This criterion is utilized solely for deciding whether to freeze the acyclicity penalty coefficient $\gamma$. it is not employed for any kind of thresholding of the final output graph. Small values of $\eta(d)$ would cause premature freezing of the penalty, possibly resulting in an excessively sparse instantaneous graph, while large values of $\eta(d)$ would prolong freezing times, thereby potentially over-penalizing the optimization process. We observed that our approach was robust even when subjected to minor variations in this retained-fraction criterion. A related sensitivity study on the Stage~1 retained fraction is described in Appendix~\ref{app:masking_ablation}, where the instantaneous error remains comparatively stable across different thresholds.
\begin{algorithm}
    \caption{Dynamic Stable Differentiable Causal Discovery}
    \label{alg:dynamic_sdcd}
    \begin{algorithmic}
    \STATE{\bfseries Input:}Multivariate time series $\{X_t\}_{t=1}^T$, lag order $L$, sparsity parameter $\lambda$
    \STATE{\bfseries Output: }Lagged adjacency matrices $\{A_\ell\}_{\ell=1}^L$, instantaneous DAG $B$

    \vspace{0.3em}
    \STATE \textbf{Stage 1: Node-wise temporal preselection}
    \FOR{$j = 1$ to $d$}
      \FOR{epoch $=1$ to $E_1$}
        \STATE Minimize using Adam Optimizer
          \[
            \mathcal{J}_j(\theta) = -\frac{1}{n}
            \sum_t \log p_\theta(X_{t+1}^j \mid X_{\mathcal{V}_t^{(j)}})
            + \lambda \|\theta\|_1
          \]
      \ENDFOR
      \STATE Compute edge scores as group norms of the first layer input weights
    \ENDFOR
    \STATE Generate masks $\{\mathcal{M}_{A_\ell}\}, \mathcal{M}_B$ by thresholding

    \vspace{0.3em}
    \STATE \textbf{Stage 2: Constrained optimization for structure refinement}
    \STATE Initialize $(\{A_\ell\}, B)$ using Stage~1
    \STATE $\gamma \leftarrow 0$, $\mathrm{frozen}\leftarrow\mathrm{false}$
    
    \FOR{epoch $=1$ to $E_2$}
      \STATE Update $(\{A_\ell\}, B)$ using Adam Optimizer
      \[
        \mathcal{L} =
        \mathcal{L}_{\mathrm{nll}}
        + \alpha \sum_\ell \|A_\ell\|_1
        + \beta \|B\|_1
        + \gamma \rho(|B|)
        + \lambda_{\mathrm{2c}} \|B \odot B^\top\|_1
      \]
      \STATE Apply masks:
      \[
      A_\ell \leftarrow A_\ell \odot \mathcal{M}_{A_\ell},
      \quad \ell=1,\ldots,L, 
      \quad
      B \leftarrow B \odot \mathcal{M}_B .
      \]
      \IF{$\mathrm{frozen}=\mathrm{false}$}
        \STATE Increase $\gamma$ according to the penalty schedule.
        \STATE Extract a candidate instantaneous DAG $\widehat{G}_B$ from $B$.
        \IF{$\widehat{G}_B$ is acyclic and $|E(\widehat{G}_B)| \ge E_{\min}$}
          \STATE Freeze $\gamma$ and set $\mathrm{frozen}\leftarrow\mathrm{true}$.
        \ENDIF
       \ENDIF
    \ENDFOR
    \STATE{\bfseries Return:}$\{A_\ell\}_{\ell=1}^L$, $B$
    \end{algorithmic}
\end{algorithm}

In the reported experiments, we use Adam optimization with different learning rate for both stages. Stage~1 is trained for $E_1=200$ epochs with sparsity coefficient $\lambda=0.15$. In stage~2, we use learning rate $\eta_2 = 1.25\times10^{-3}$, $\alpha=0.02$, $\beta=0.001$, and $\lambda_{\mathrm{2c}}=0.05$, with a linear schedule for $\gamma$ starting from zero. Penalty freezing is triggered when the extracted instantaneous DAG is acyclic and retains at least $\lfloor \eta(d) s_{\mathrm{inst}} d \rfloor$ edges, where $s_{\mathrm{inst}}$ denotes the expected instantaneous indegree.
 
\subsection{Computational Complexity}
\label{sec:complexity}

Let $n = N(T-L-1)$ denote the number of training pairs after forming the temporal design window, where $N$ is the number of independent trajectories, $T$ is the time horizon, and $L$ is the maximum lag order. Using the node-wise predictors of hidden width $H$, Stage-1 screening requires 
\begin{equation}
\mathcal{O}\!\left(E_1\, n\, d\, H\, (dL + d_{\mathrm{inst}})\right)
\end{equation}
time, where $E_1$ is the number of epochs of Stage~1 and $d_{\mathrm{inst}}=d$ when instantaneous condition is used and $0$ otherwise. Stage-2 operates on the masked inputs and costs 
\begin{equation}
\mathcal{O}\!\left(E_2\, n\, d\, H\, (\rho_{\mathrm{lag}} dL + \rho_{\mathrm{inst}} d) \right) ,
\end{equation}
$\rho_{\mathrm{lag}}$ and $\rho_{\mathrm{inst}}$ represent the retained
fractions of lagged and instantaneous candidates after Stage~1. The spectral acyclicity penalty contributes an overhead of 
\begin{equation}
\mathcal{O}\!\left(E_2\, (n/B_{\mathrm{batch}})\, K d^2\right)
\end{equation}
where $B_{\mathrm{batch}}$ is the mini-batch size and $K$ denotes the number of power iteration steps. Combining these terms, the overall computational complexity is
\begin{equation}
\label{eq:overall_complexity}
\begin{aligned}
\mathcal{O}~\bigl(
& E_1 n d H (dL+d_{\mathrm{inst}})
  + E_2 n d H(\rho_{\mathrm{lag}}dL+\rho_{\mathrm{inst}}d) \\
& \quad + E_2 (n/B_{\mathrm{batch}}) K d^2
\bigr).
\end{aligned}
\end{equation}

Complete details are provided in Appendix~\ref{app:complexity}.

\section{Numerical Results}
\label{sec:results}

\subsection{Performance with respect to Dimension}

\begin{figure*}
  \vskip 0.15in
  \begin{center}
    \centerline{
      \includegraphics[width=0.9\columnwidth]{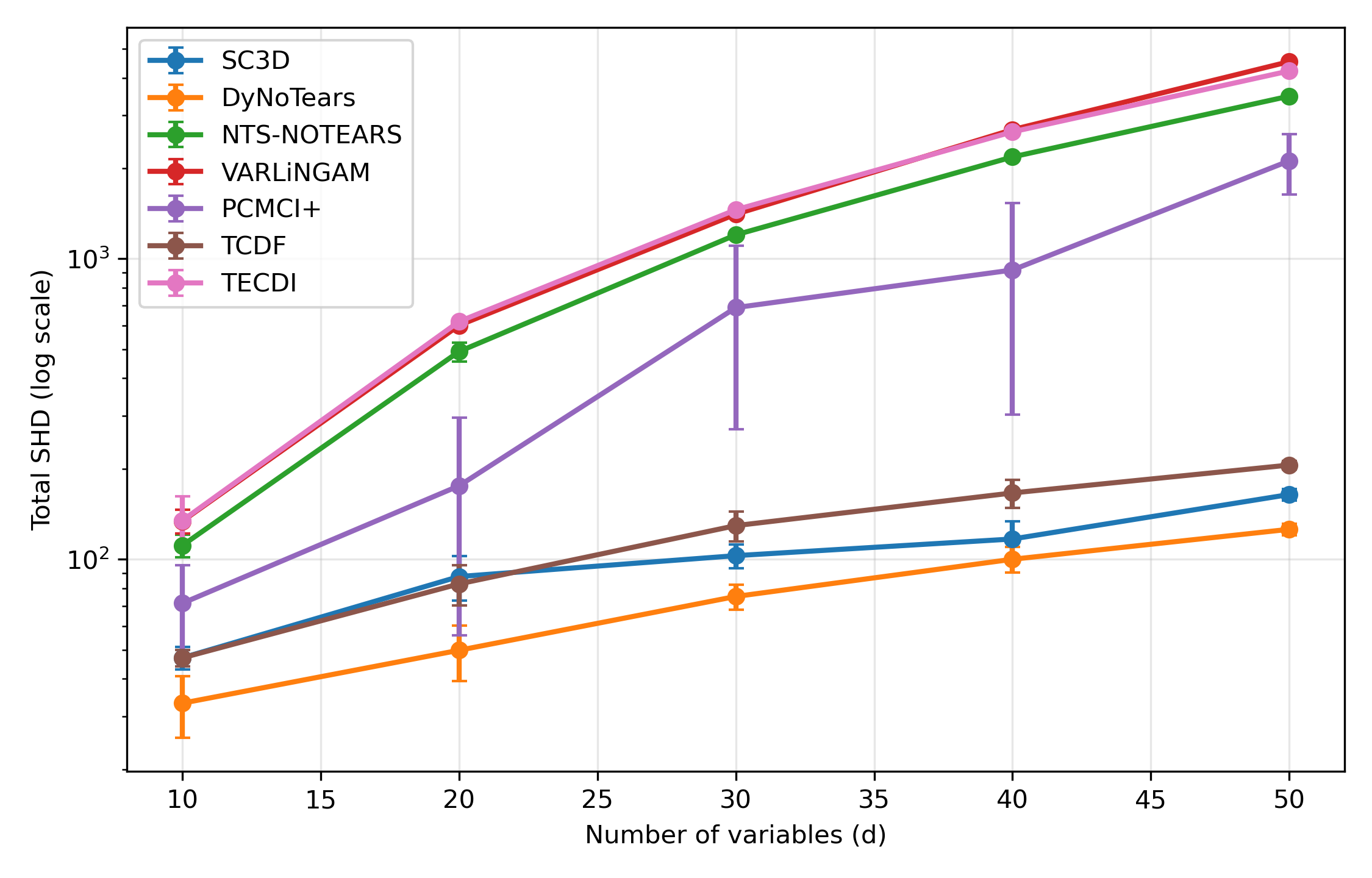}
      \hfill
      \includegraphics[width=0.9\columnwidth]{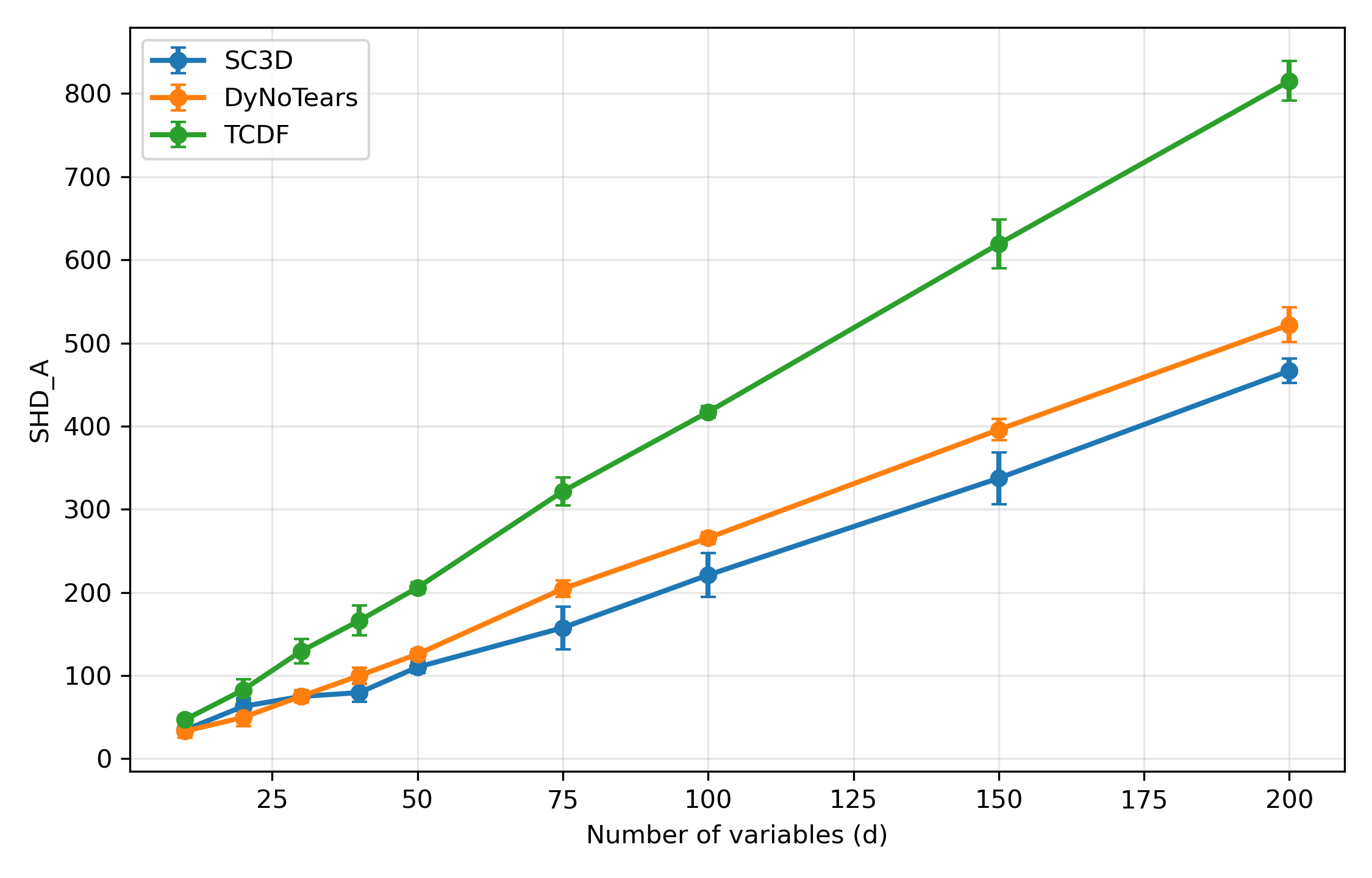}
    }
    \caption{
    Performance of causal structure recovery under increasing dimension ($L=3$, $T=200$).
    \textbf{Left:} Total structural Hamming distance $\mathrm{SHD}_{\mathrm{total}}$, accounting for both lagged and instantaneous errors, shown on a logarithmic scale.
    \textbf{Right:} Lagged-only structural error $\mathrm{SHD}_A$ shown on a linear scale for methods which model lagged dependencies, enabling a direct comparison for temporal structure recovery.
    }
    \label{fig:shd_scalability}
  \end{center}
  \vskip -0.3in
\end{figure*}

We analyze how causal structure recovery methods scale with number of variables under the nonlinear SVAR model in ~\eqref{eq:linear_svar}. Specifically, we set the lag order to $L=3$ and time horizon $T=200$, vary the dimension given from the set $d \in \{10,20,30,40,50\}$, while keeping the expected graph sparsity constant. Additionally, we average the results across five seeds to compare the performance across different baselines. \cref{fig:shd_scalability} (left) presents our plots the total structural Hamming distance, 
\[ \mathrm{SHD}_{\mathrm{total}} = \sum_{\ell=1}^{L}
\mathrm{SHD}\!\left(A_\ell, A_\ell^\star\right)
+ \mathrm{SHD}\!\left(B, B^\star\right), \]
where $\mathrm{SHD}(G,G^\star)$ denotes the structural Hamming
distance between the estimated directed adjacency support $G$ and the ground-truth support $G^\star$. The metric is shown on a logarithmic scale to accommodate the large dynamic range across methods and dimensions. For methods that do not provide an instantaneous adjacency matrix under this evaluation protocol, we set $B \equiv 0$ and count all true instantaneous edges as skipped for methods that do not model instantaneous causality (eg., DYNOTEARS, PCMCI+). Given this unified evaluation, \textsc{SC3D} has the lowest or comparable total SHD performance across all dimensions and significantly outperforms a number of temporal baselines whose total structural error grow rapidly as the dimension $d$ increases. While DYNOTEARS yields lower SHD in certain scenarios, it can be regarded as a more restricted temporal differentiable benchmark due to its linearity assumptions.

To separate the scalability with regards to temporal lag only structure, \cref{fig:shd_scalability} (right) illustrates the lag-only error $\mathrm{SHD}_A$ for \textsc{SC3D} and DYNOTEARS in a linear scale to earmark absolute differences in lag recovery. Even though \textsc{SC3D}, DYNOTEARS and TCDF outperform the classical baselines in terms of scalability, it is clearly evident that \textsc{SC3D} consistently achieves lower $\mathrm{SHD}_A$ even when dimensions increase which showcases more accuracy in recovery of lagged dependencies. Moreover, the low variance across random seeds indicate stable optimization and structural recovery. In summary, we can say that \textsc{SC3D} demonstrated robust lagged structure learning along with stable handling of instantaneous effects, which accounts for its exceptional scalability to high dimensional regimes. 

\subsection{Ranking performance via AUROC and AUPRC}
\label{sec:results_auroc}

\begin{table*}[t]
  \caption{Aggregated lagged and instantaneous AUROC/AUPRC at $d=30$ ($L=3$, $T=200$).
  Reported values are mean $\pm$ standard deviation over five random seeds.
  Methods for which an explicit instantaneous score  matrix is not reported under this evaluation protocol are marked as ``\xmark'' for the corresponding metrics.}
  \label{tab:auroc_auprc}
  \centering
  \begin{small}
    \begin{sc}
      \begin{tabular}{lcccc}
        \toprule
        Method & AUROC$_A \uparrow$ & AUPRC$_A \uparrow$ & AUROC$_B \uparrow$ & AUPRC$_B \uparrow$ \\
        \midrule
        \textsc{SC3D} & \textbf{0.910 $\pm$ 0.009} & 0.766 $\pm$ 0.028 &
        0.811 $\pm$ 0.036 & \textbf{0.715 $\pm$ 0.061} \\
        DYNOTEARS & 0.858 $\pm$ 0.022 & 0.722 $\pm$ 0.044 & \xmark & \xmark \\
        NTS-NOTEARS & 0.845 $\pm$ 0.026 & 0.736 $\pm$ 0.031 & 0.780 $\pm$ 0.051
        & 0.393 $\pm$ 0.051 \\
        PCMCI+ & 0.890 $\pm$ 0.022 & \textbf{0.778 $\pm$ 0.029} & \xmark & \xmark \\
        TCDF & 0.503 $\pm$ 0.006 & 0.050 $\pm$ 0.010 & \xmark & \xmark \\
        TECDI & 0.904 $\pm$ 0.010 & 0.740 $\pm$ 0.031 & \textbf{0.943 $\pm$ 0.011}
        & 0.437 $\pm$ 0.026 \\
        NeuralGC (cMLP) & 0.500 $\pm$ 0.000 & 0.047 $\pm$ 0.006 & \xmark & \xmark \\
        \bottomrule
      \end{tabular}
    \end{sc}
  \end{small}
  \vskip -0.05in
\end{table*}

Next, we assess the ranking quality of the estimated causal relationships using certain threshold free evaluation metrics: area under the receiver operating characteristics curve (AUROC) and area under the precision-recall curve (AUPRC). These evaluation metrics are particularly meaningful if the underlying causal graph is sparse. For lagged causal relationships, we compute the aggregated scores obtained by considering the maximum absolute edge scores across all lags for every ordered variable pair and comparing it to the ground truth presence of any lagged edge. The instantaneous AUROC and AUPRC is reported only for methods that output an explicit instantaneous score matrix under the evaluation protocol. \cref{tab:auroc_auprc} summarizes our results on a representative dimension $d=30$.

As can be seen in \cref{tab:auroc_auprc}, \textsc{SC3D} is found to generate strong lagged ranking performance, with the best AUROC and near best AUPRC comparable to or exceeding existing temporal baselines such as PCMCI+ and TECDI. 
The performance of TCDF and NeuralGC (cMLP) in terms of lagged AUROC is nearly random in this experiment setup, which suggests that these models have problems ranking lagged causation relations with the synthetic dynamic SEM benchmark. As for the instantaneous structure, TECDI obtains the best AUROC$_B$ score, while \textsc{SC3D} obtains the best AUPRC$_B$ score. Since AUPRC is especially indicative when dealing with sparse networks, these results indicate that \textsc{SC3D} ranks instantaneous edges correctly while achieving decent precision at the same time. Overall, the above results show that \textsc{SC3D} provides a well balanced ranking for both lagged and instantaneous dependencies while maintaining competitive performance over specialized baselines.

\subsection{Sensitivity to Lag Order}
\label{sec:results_L_sweep}

\begin{figure}
  \centering
  \includegraphics[width=\columnwidth]{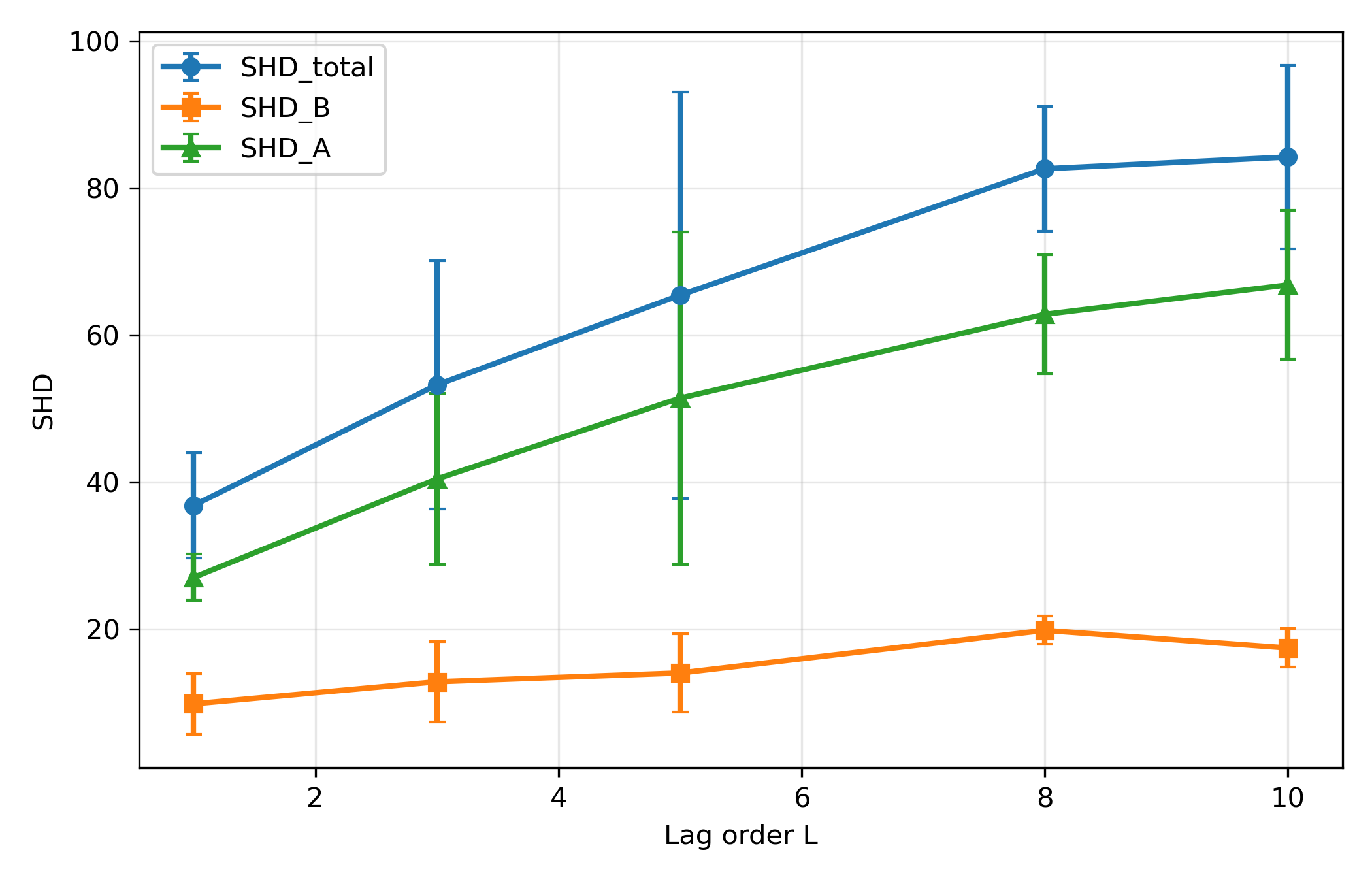}
  \caption{Sensitivity of \textsc{SC3D} to lag order $L$ ($d=8$, $T=200$). The total structural error increases gradually as $L$ grows, while the instantaneous component remains well controlled, indicating stable optimization and robustness to increasing temporal depths.}
  \label{fig:shd_vs_L}
  \vskip -0.25in
\end{figure}

We now analyze how the proposed \textsc{SC3D} framework behaves as we increase the maximum lag order $L$ in the framework. This analysis is not intended to compare scalability in terms of dimensions to other baselines models since most competing methods either do not explicitly model large lag orders or tend to be computationally infeasible beyond a small value of $L$. Instead, we aim to evaluate the robustness of SC3D to increasing temporal depths. We fix the number of variables to $d=8$ and time horizon $T=200$, and sweep over the lag order $L \in \{1,3,5,8,10\}$. We aggregated the results over five random seeds for each $L$. We plot the total SHD, along with the decomposition into lagged and instantaneous components.

\cref{fig:shd_vs_L} shows that the total SHD increases gradually with an increase in $L$, reflecting the larger number of the candidate temporal edges as $L$ increases. On the other hand, the majority of such increment comes from the lagged part, with the instantaneous part still being relatively stable during the entire sweep. This confirms that Stage~1 temporal screening continues to control the much bigger lagged search space, and enforcing acyclicity only on the instantaneous block on the instantaneous component remains numerically stable as additional lagged relationships are introduced. In general, the test demonstrates the scalability of \textsc{SC3D} with respect to the moderately large depth levels without numerical instability or significant deterioration in structure recovery.

\subsection{Qualitative Graph Recovery}
\label{sec:qualitative_analysis}

\begin{figure*}
  \centering
  \includegraphics[width=0.68\textwidth]{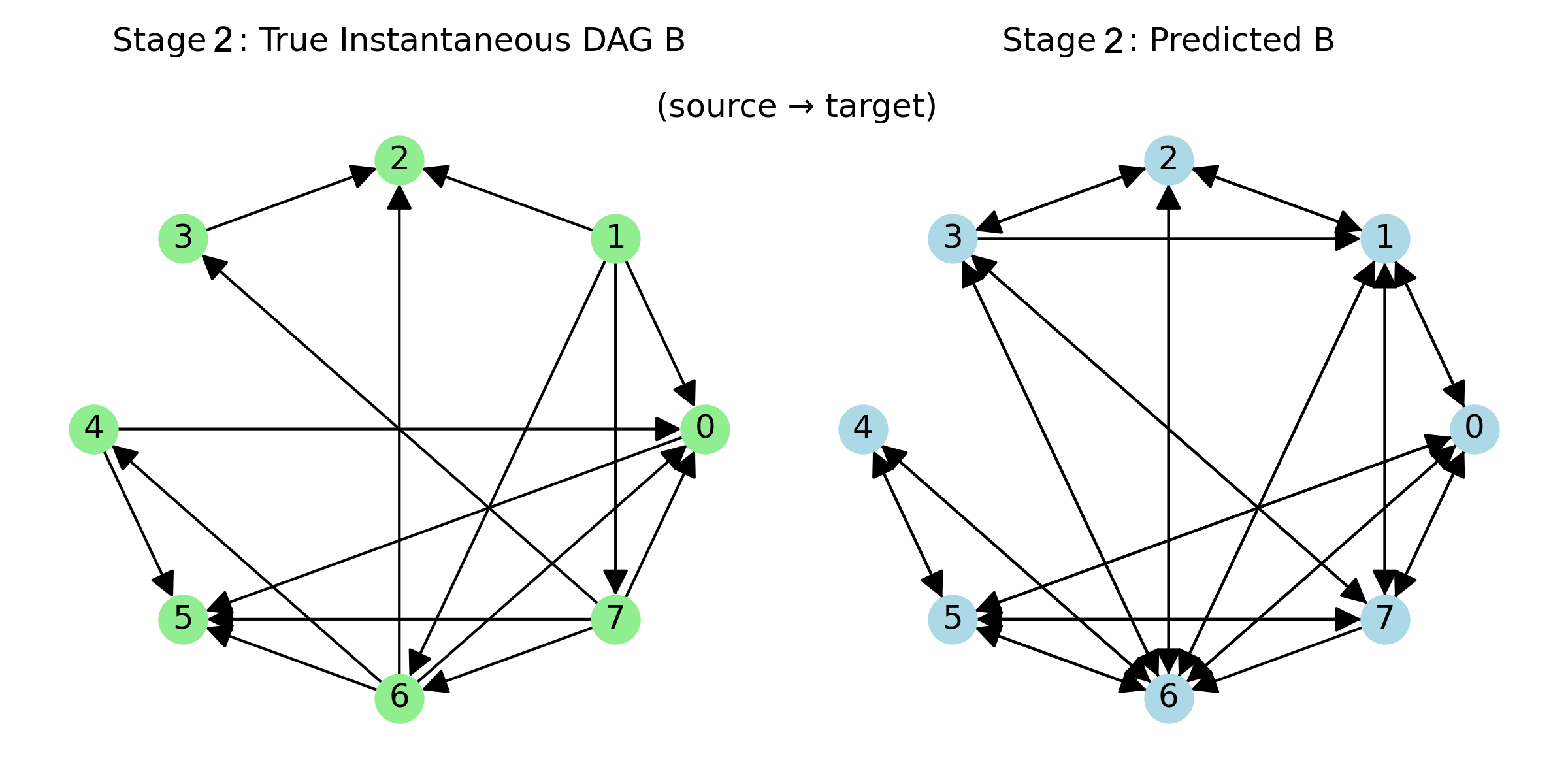}

  \vspace{0.6em}

  \includegraphics[width=0.92\textwidth]{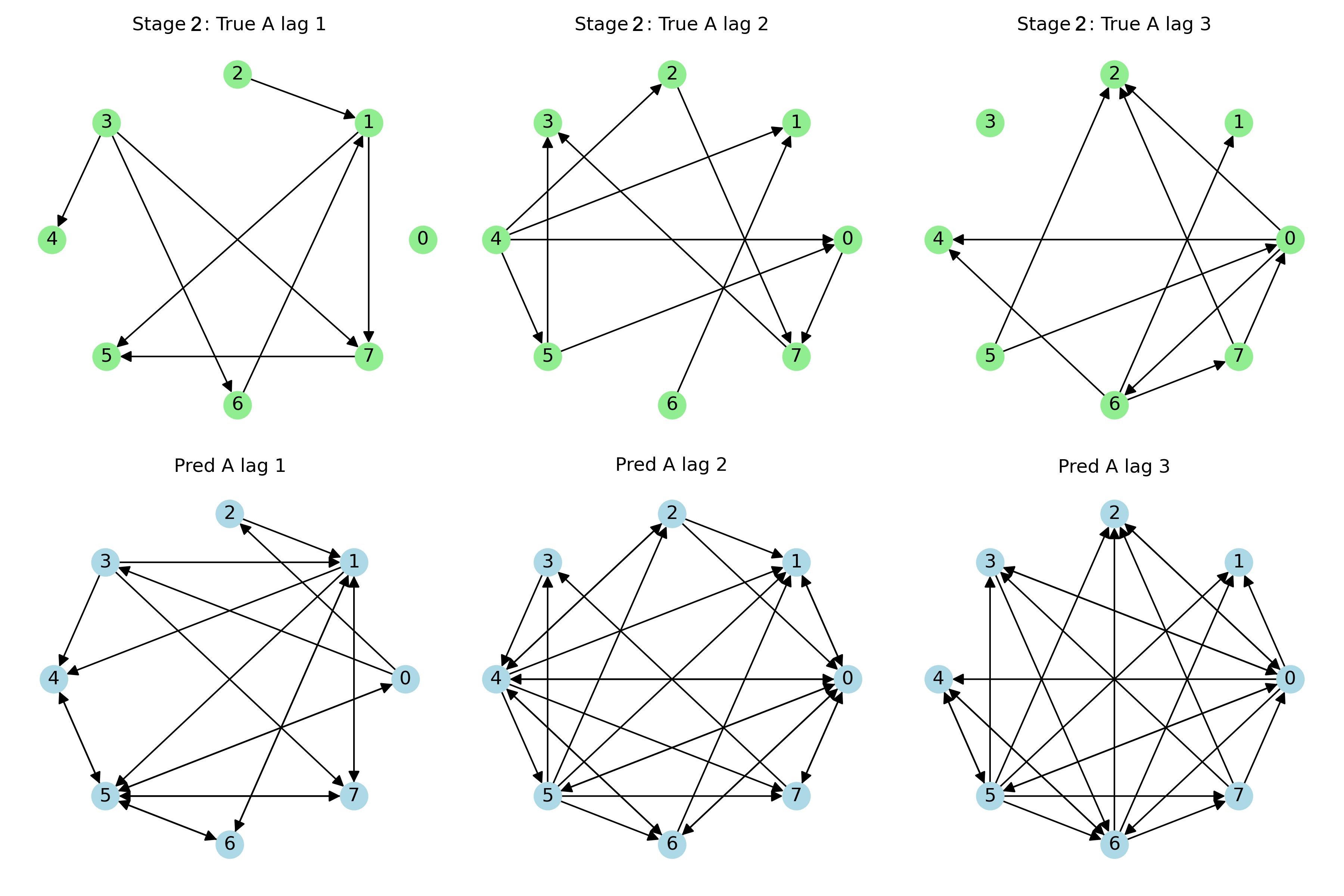}

  \caption{
  Recovering the instantaneous and lagged causal structure on a
  representative synthetic dynamic SEM/SVAR instance.
  \textbf{Top:} ground-truth instantaneous DAG $B^\star$ and the
  instantaneous graph recovered by \textsc{SC3D} after acyclicity enforcement.
  \textbf{Bottom:} ground-truth and recovered lag-specific causal graphs
  $\{A_\ell\}_{\ell=1}^{3}$, with each column corresponding to one lag.
  Directed edges are shown from source to target. The recovered graphs capture
  several dominant dependencies while also revealing the types of spurious
  edges reflected in the quantitative SHD metrics.}
  \label{fig:qualitative_recovery}
\end{figure*}

To understand the overall performance of \textsc{SC3D}, we also visualize the reconstructed causal graphs for a representative synthetic dynamic SEM/SVAR example. \cref{fig:qualitative_recovery} illustrates the ground truth and reconstructed instantaneous graph along with the temporal graphs at each time lag. It can be seen that the reconstructed graph retains several important intra-slice relationships from the truth but is also acyclic after the Stage 2 refinement process. This highlights the importance of the spectral acyclicity constraint and the two-cycle penalty in controlling the instantaneous structure.

When we consider lagged components, the reconstructed graphs show many of the dominant temporal relations for several lag times, although some spurious relations still exist. This behavior aligns with our quantitative analysis, since \textsc{SC3D} does not always reconstruct every edge correctly but provides stable reconstructed graphs which preserve important temporal and instantaneous structures. In this regard, the qualitative analysis helps us understand the graphs that were learned by the algorithm complements the AUROC/AUPRC and SHD evaluation metrics.

\subsection{Performance across dynamical systems}
\label{sec:results_datasets}

We evaluate the performance of the \textsc{SC3D} on different benchmark and synthetic dynamical systems intended to investigate various aspects of temporal causal discovery such as chaotic dynamics, nonstationarity and nonlinear long-range dependencies. Details of these datasets and governing equations are mentioned in Appendix~\ref{app:datasets}.

\subsubsection{Lorenz-96 chaotic dynamics}
\label{sec:results_lorenz96}
\vspace{0.25cm}
The performance is next evaluated in Lorenz-96 system, a commonly used chaotic benchmark with sparse local interactions. Here, the causal relationships are purely lagged with no instantaneous effects and the ground truth reveals fixed indegree pattern. Thus, we study the lagged recovery in terms of top-$k$ approach with $k=3$ incoming edges per target, along with aggregated AUROC and AUPRC capturing threshold independent ranking performance.

\begin{table*}[t]
  \caption{Lorenz-96 (chaotic dynamics): lagged structure recovery at $d=20$ with lag order $L=1$ and horizon $T=200$ (top-$k$ evaluation with $k=3$). Reported values are mean $\pm$ standard deviation over five random seeds. Since Lorenz-96 contains no instantaneous effects, all metrics are reported for the lagged component only.}
  \label{tab:lorenz96}
  \centering
  \begin{small}
    \begin{sc}
      \begin{tabular}{lccc}
        \toprule
        Method & SHD$_A$ (top-$k$) $\downarrow$ & AUROC$_A$ (agg) $\uparrow$ & AUPRC$_A$ (agg) $\uparrow$\\
        \midrule
        \textsc{SC3D}     & \textbf{35.6 $\pm$ 1.5}  & \textbf{0.833} $\pm$ 0.000 & \textbf{0.719} $\pm$ 0.000 \\
        PCMCI+            & 66.4 $\pm$ 12.1 & 0.739 $\pm$ 0.055 & 0.438 $\pm$ 0.088 \\
        DYNOTEARS         & 102.0 $\pm$ 0.0  & 0.500 $\pm$ 0.000 & 0.158 $\pm$ 0.000 \\
        NTS-NOTEARS       & 94.0 $\pm$ 4.6  & 0.578 $\pm$ 0.041 & 0.214 $\pm$ 0.035 \\
        TECDI             & 71.6 $\pm$ 4.3  & 0.754 $\pm$ 0.022 & 0.395 $\pm$ 0.040 \\
        NeuralGC (cMLP)              & 102.0 $\pm$ 0.0  & 0.500 $\pm$ 0.000 & 0.157 $\pm$ 0.000 \\
        \bottomrule
      \end{tabular}
    \end{sc}
  \end{small}
  \vskip -0.2in
\end{table*}

\cref{tab:lorenz96} presents the results at a representative dimension $d=20$. \textsc{SC3D} attains the lowest top-$k$ structural error among all baselines while exhibiting significantly lower variance across seeds. Moreover, it attains the highest AUROC and AUPRC, showcasing accurate and stable ranking of the true causal parents in this chaotic regime. Among the baseline methods, the two best-performing algorithms are PCMCI+ and TECDI, yet both obtain substantially larger SHD$_A$ and lower AUPRC values compared to \textsc{SC3D}. In contrast, NTS-NOTEARS performs better than the baseline algorithms (DYNOTEARS and NeuralGC), whose AUROC values are close to random ranking, yet it still lags trails the strongest methods in termns of both structural recovery and ranking performance. Overall these results showcase \textsc{SC3D} remains robust under challenging lag-only dynamics while also maintaining scalability and stability. 

\subsubsection{Time Varying Structural Equation Model (TVSEM).}
\label{sec:results_tvsem}

Next, we evaluate the performance of \textsc{SC3D} on a time-varying structural equation model (TVSEM), a low-dimensional nonstationary benchmark where the direction of causality reverses over time. This system oscillates between two regimes in non-overlapping time intervals where the dominant lagged dependency switches between $X_{t-1}\!\rightarrow\!Y_t$ and $Y_{t-1}\!\rightarrow\!X_t$. This example dataset explicitly tests the ability of our method, when applied in a windowed manner, to track a changing direction of causality over time rather than just averaging them out. \cref{fig:tvsem_tracking} depicts the scores for the Lag-1 edge for all windows considered in the task. In the left panel, we show the results for SC3D while the right panels puts together the baselines (DYNOTEARS and PCMCI+). The shaded regions in the plots denote the first regime and the unshaded regions denote switched or second regime.

\begin{figure}
  \centering
  \includegraphics[width=\columnwidth]{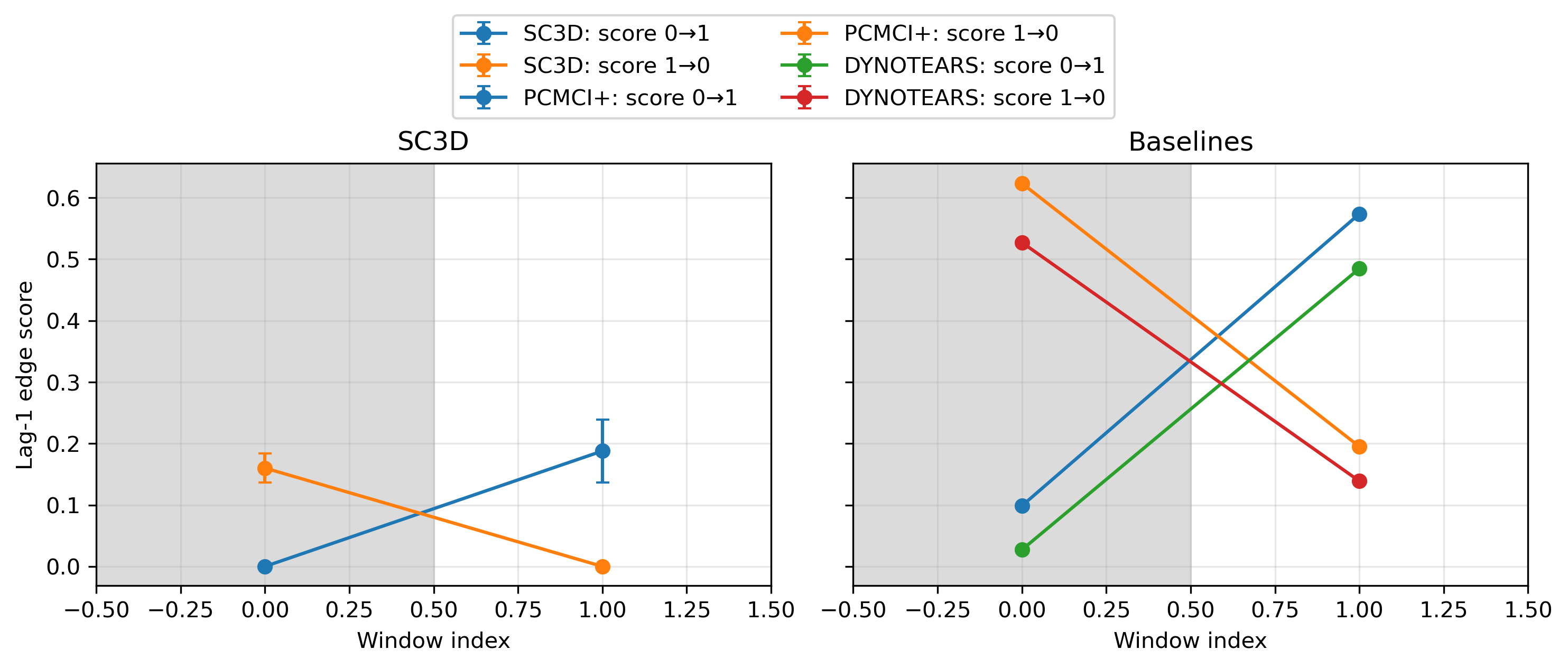} 
  \caption{
  TVSEM regime switching direction tracking using windowed evaluation (lag order $L=1$). Left: \textsc{SC3D}. Right: baselines (DYNOTEARS and PCMCI+). \textsc{SC3D} displays smoother and temporally consistent score trajectories that fit well to regime changes while baselines show sharper but less stable transitions throughout windows.
  }
  \label{fig:tvsem_tracking}
  \vskip -0.25in
\end{figure}

\textsc{SC3D} showcases a distinct and consistent separation of the two window-level directional scores, with the dominant edge smoothly adapting to the change of regime. On the contrary, baseline methods display higher score magnitudes but exhibit less consistent behavior across the regimes, which indicates reduced sensitivity to temporal changes in directionality. Although all the methods attain perfect window-level directional accuracy in this low-dimensional setup, the score trajectories reveal significant differences. \textsc{SC3D} has more stable and interpretable directional information over time whereas the baselines tend to depend on sharper but more rigid score transitions. Therefore, this experiment may only be viewed as a qualitative analysis of temporal score stability, rather than as a strict comparison of accuracy. The above results reflect the advantage of \textsc{SC3D} in providing stable and interpretable directional information when applied to nonstationary time-series systems.

\subsubsection{Nonlinear Continuous 8-variable (NC8) System}

\begin{table}[ht]
  \caption{Nonlinear Continuous 8-variable dataset (NC8): aggregated ranking performance for lagged causal structure ($d=8$, $L=4$, $T=200$). NC8 contains no instantaneous effects, so all metrics correspond to lagged structure only. Reported values are mean $\pm$ standard deviation over five random seeds.}
  \label{tab:nc8}
  \centering
  \small
  \begin{tabular}{lcc}
    \toprule
    Method & AUROC$_A^{\mathrm{agg}} \uparrow$ & AUPRC$_A^{\mathrm{agg}} \uparrow$ \\
    \midrule
    \textsc{SC3D}      
      & 0.885 $\pm$ 0.004 & \textbf{0.828 $\pm$ 0.007} \\
    PCMCI+             
      & \textbf{0.900 $\pm$ 0.008} & 0.804 $\pm$ 0.025 \\
    NTS-NOTEARS        
      & 0.723 $\pm$ 0.027 & 0.374 $\pm$ 0.035 \\
    DYNOTEARS          
      & 0.653 $\pm$ 0.021 & 0.410 $\pm$ 0.007 \\
    TECDI              
      & 0.525 $\pm$ 0.082 & 0.300 $\pm$ 0.052 \\
    NeuralGC (cMLP)    
      & 0.500 $\pm$ 0.000 & 0.161 $\pm$ 0.000 \\
    \bottomrule
  \end{tabular}
  \vskip -0.1in
\end{table}

Next, we consider lagged causal discovery on the NC8 dataset, a low dimensional nonlinear dynamical system ($d=8$) which has structured long-range temporal edges without any instantaneous edges and a maximum lag order $L=4$. The graph in this setting only has lagged edges and no instantaneous causal effects. Thus we evaluate the lagged components using threshold-free methods for ranking that are averaged over time lags. 

As seen in ~\cref{tab:nc8}, the PCMCI+ has the best AUROC score, reflecting excellent performance in global ranking performance on this small-scale system. However, \textsc{SC3D} has the best AUPRC score, which is typically informative on sparse causal graphs, where precision at high confidence is crucial. The NTS-NOTEARS algorithm improves over DYNOTEARS and NeuralGC in terms of AUROC, although its AUPRC value is significantly lower than those of \textsc{SC3D} and PCMCI+. TECDI and NeuralGC exhibit inferior ranking performance in this observational lag-only setup. In summary, these findings indicate that \textsc{SC3D} has competitive performance on nonlinear long range temporal dependencies, especially with strong precision-recall behavior for sparse lagged causal structure. 

\subsection{Real-world validation on CausalRivers}

\begin{table*}[t]
\caption{Real-world CausalRivers-East Germany benchmark. We sample connected
subgraphs of size $d=10$ and $d=20$ and evaluate lagged structure recovery only,
since no instantaneous ground-truth DAG is available. Reported values are mean
$\pm$ standard deviation over five random subgraph seeds.}
\label{tab:causalrivers}
\centering
\begin{small}
\begin{tabular}{lcccc}
\toprule
Method & $d$ & SHD$_A$ (top-$E$) $\downarrow$ & AUROC$_A$ (agg) $\uparrow$ & AUPRC$_A$ (agg) $\uparrow$ \\
\midrule
\textsc{SC3D} & 10 & $16.00 \pm 2.00$ & $\mathbf{0.6095 \pm 0.1097}$ & $0.2203 \pm 0.1129$ \\
DYNOTEARS & 10 & $81.00 \pm 0.00$ & $0.5085 \pm 0.0226$ & $0.1089 \pm 0.0199$ \\
NTS-NOTEARS & 10 & $27.80 \pm 29.99$ & $0.5522 \pm 0.1285$ & $0.2088 \pm 0.1790$ \\
PCMCI+ & 10 & $15.60 \pm 2.19$ & $0.5903 \pm 0.1052$ & $\mathbf{0.2488 \pm 0.1027}$ \\
TECDI & 10 & $\mathbf{15.20 \pm 2.28}$ & $0.5566 \pm 0.0954$ & $0.1552 \pm 0.0442$ \\
Neural-GC & 10 & $80.60 \pm 0.89$ & $0.5000 \pm 0.0000$ & $0.1044 \pm 0.0099$ \\
TCDF & 10 & $68.20 \pm 28.07$ & $0.4890 \pm 0.0108$ & $0.1020 \pm 0.0051$ \\
\midrule
\textsc{SC3D} & 20 & $35.20 \pm 1.79$ & $0.6228 \pm 0.0537$ & $0.0910 \pm 0.0216$ \\
DYNOTEARS & 20 & $359.60 \pm 0.55$ & $0.5042 \pm 0.0102$ & $0.0632 \pm 0.0206$ \\
NTS-NOTEARS & 20 & $95.80 \pm 148.29$ & $\mathbf{0.6360 \pm 0.1266}$ & $\mathbf{0.1681 \pm 0.1119}$ \\
PCMCI+ & 20 & $35.00 \pm 3.83$ & $0.6114 \pm 0.0250$ & $0.1317 \pm 0.0392$ \\
TECDI & 20 & $\mathbf{34.40 \pm 0.89}$ & $0.5668 \pm 0.0484$ & $0.1442 \pm 0.0296$ \\
NeuralGC (cMLP) & 20 & $360.60 \pm 0.55$ & $0.5000 \pm 0.0000$ & $0.0511 \pm 0.0014$ \\
TCDF & 20 & $360.80 \pm 0.45$ & $0.5064 \pm 0.0146$ & $0.0582 \pm 0.0106$ \\
\bottomrule
\end{tabular}
\end{small}

\end{table*}

We evaluate our approach with the CausalRivers benchmark~\cite{stein2025causalrivers}, which utilizes real-world time-series data extracted from river discharge measurements. The available ground truth in the dataset corresponds to hydrological connectivity rather than an instantaneous intra-slice DAG, thus we test the lagged structure recovery only. Specifically, for each run, we randomly sample connected subgraphs of size $d \in \{10,20\}$ from the East Germany region and partition time series into disjoint windows, each viewed as an independent trajectory. We use SHD at top-$E$, where $E$ is the number of ground truth edges, together with the aggregated AUROC and AUPRC over the lagged edges.

As shown in Table~\ref{tab:causalrivers}, \textsc{SC3D} exhibits competitive performance for both values of graph size. When $d=10$, \textsc{SC3D} obtains the highest AUROC and is competitive in structural accuracy against TECDI and PCMCI+. Meanwhile, when $d=20$, \textsc{SC3D} still achieves favorable SHD while exhibiting stable performance among different seeds. On the contrary, DYNOTEARS, Neural-GC, and TCDF degrade drastically in structural recovery with increasing $d$. In addition, NTS-NOTEARS obtains strong rankings but with substantial variance. Overall, these experimental results suggest that \textsc{SC3D} achieves competitive performance on real-world temporal causal data beyond controlled synthetic benchmarks.

\subsection{Ablation Study}
\label{sec:ablation}

\begin{table*}[ht]
\label{tab:ablation}
\centering
\small
\setlength{\tabcolsep}{6pt}
\renewcommand{\arraystretch}{1.15}
\caption{Ablation study for Dynamic-SC3D on VAR data. Each component (Stage 1 preselection, penalty freezing, and 2-cycle penalty) is removed to analyze its contribution to structural recovery.}
\label{tab:ablation_var}
\begin{tabular}{lccc}
\toprule
\textbf{Variant} & \textbf{SHD$_\text{total}$ $\downarrow$} & \textbf{SHD$_B$ $\downarrow$} & \textbf{F1$_B$ $\uparrow$} \\
\midrule
\multicolumn{4}{l}{\textbf{$d=20$}} \\
\midrule
Full (SC3D)                & 87.60 $\pm$ 13.17  & \textbf{24.60 $\pm$ 4.63}  & 0.510 $\pm$ 0.046 \\
Linear predictor           & \textbf{79.00 $\pm$ 10.97} & 25.20 $\pm$ 5.23 & \textbf{0.586 $\pm$ 0.069} \\
No freezing                & 97.40 $\pm$ 16.13  & 25.60 $\pm$ 6.97 & 0.505 $\pm$ 0.088 \\
No 2-cycle penalty         & 94.20 $\pm$ 12.70  & 25.60 $\pm$ 3.77 & 0.487 $\pm$ 0.037 \\
No Stage 1                 & 637.60 $\pm$ 6.62  & 133.80 $\pm$ 15.05 & 0.193 $\pm$ 0.021 \\
\midrule
\multicolumn{4}{l}{\textbf{$d=40$}} \\
\midrule
Full (SC3D)                & \textbf{116.80 $\pm$ 15.04} & \textbf{37.20 $\pm$ 6.52} & \textbf{0.730 $\pm$ 0.075} \\
Linear predictor           & 157.60 $\pm$ 19.89 & 60.00 $\pm$ 5.18  & 0.506 $\pm$ 0.046 \\
No freezing                & 148.80 $\pm$ 42.13 & 46.00 $\pm$ 12.05 & 0.618 $\pm$ 0.130 \\
No 2-cycle penalty         & 165.40 $\pm$ 11.50 & 51.80 $\pm$ 5.64  & 0.506 $\pm$ 0.058 \\
No Stage 1                 & 2848.20 $\pm$ 45.20 & 658.20 $\pm$ 51.46 & 0.105 $\pm$ 0.019 \\
\bottomrule
\end{tabular}
\end{table*}

We perform an ablation study to analyze the how the key components of \textsc{SC3D} operate on lower ($d=20$) and large ($d=40$) dimensions. \cref{tab:ablation_var} presents the comparison between the full proposed method against ablated variants that remove Stage~1 preselection, remove penalty freezing, remove the 2-cycle penalty, or replace the nonlinear predictor with a linear predictor. We observe that removing Stage~1 causes the most significant degradation: the total SHD increases substantially at both dimensions, and the instantaneous F1 score becomes close to zero. This results shows that the preselection stage is crucial phase for controlling the search space before deploying the constrained refinement.

It can be seen from the other ablation settings that acyclic penalty and the 2-cycle penalty also contribute in obtaining stable recovery of the graph at the instantaneous level, particularly when the dimension is high. The linear predictor performs well when $d=20$ since it has an approximately linear structure in the case of a synthetic SVAR setting. However, its performance deteriorates when $d=40$, where the complete nonlinear \textsc{SC3D} with all the components gives the best recovery. This suggests that the individual components of the \textsc{SC3D} method become increasingly important with increasing graph dimensions. Additional sensitivity analysis regarding the Stage~1 masking thresholds is provided in Appendix~\ref{app:masking_ablation}.

\section{Conclusion}
\label{sec:conclusion}

We presented \textsc{SC3D}, a stable, scalable and differentiable framework for causal discovery in multivariate dynamical systems with both lagged and instantaneous relationships. By combining node-wise temporal preselection along with constrained optimization subject to a spectral acyclicity penalty, \textsc{SC3D} facilitates scalable structure learning while avoiding numerical instabilities of smooth acyclicity constraints. The proposed formulation covers nonlinear dynamic structural equation models, with the linear SVAR model considered as a special case, and the theoretical analysis shows that the
instantaneous acyclicity constraint is sufficient to ensure acyclicity of the
time-unrolled graph. Empirical results demonstrated strong and stable performance of \textsc{SC3D} across a range of synthetic SVAR systems, nonlinear and chaotic dynamical systems, and real-world datasets. It addresses the joint recovery of lagged and instantaneous causal structure while remaining competitive with existing temporal methods. Future work will focus on explicitly time varying causal graphs, online settings that exhibit delays and instantaneous effects simultaneously. 

\section{Acknowledgements}

This material is based upon work supported by the U.S. Department of Energy (DOE), Office of Science, Office of Advanced Scientific Computing Research, under Contract No. DE-AC02–06CH11357. We acknowledge
support from DOE FES award
``DeepFusion Accelerator for Fusion Energy Sciences in Disruption Mitigation" (PI Dr. Michael Halfmoon). RM was supported by an ARO ECP award from the Program `Modeling of Complex Systems' (PM - Dr. Rob Martin). We also acknowledge the support of Penn State ICDS computing resources.

\appendix
\section{Theoretical Analysis}
\subsection{Assumptions for Stage 1}
\label{app:stage1-assumptions}

\begin{enumerate}

\item For each node $j$, the true conditional $p^\star(X_{t+1}^j \mid \mathcal{V}_t^{(j)})$ belongs to the model class $\{p_\theta(\cdot \mid \cdot)\}_\theta$.

\item The joint distribution of $(X_{t+1},X_t,\dots,X_{t+1-L})$ has a strictly positive density on its support.

\item If a variable $W\in \mathcal{V}_t^{(j)}$ belongs to the dynamic parent set of node $j$ in \eqref{eq:nonlinear_dsem} for node $j$ (either an instantaneous parent in $B^\star$ or a lagged parent in some $A_\ell^\star$),
then $X_{t+1}^j$ is not conditionally independent of $W$ given $\mathcal{V}_t^{(j)}\setminus\{W\}$.
\end{enumerate}

\subsection{Proof of \cref{thm:stage1}}
\label{app:stage1-proof}

\begin{proof}

Let us consider a node $j$ and abbreviate $\mathcal{V}:=\mathcal{V}_t^{(j)}$.
For any subset $S\subseteq \mathcal{V}$, we can write the expected conditional log-likelihood as
\begin{equation}
\begin{aligned}
    \sup_{\theta:\,\mathrm{supp}(\theta) \subseteq S} \mathbb{E}[\log p_\theta(X_{t+1}^j\mid X_S)] = C_j - \\ \inf_{\theta:\,\mathrm{supp}(\theta)\subseteq S} \mathbb{E}\!\left[\mathrm{KL}\!\left(p^\star (\cdot\mid \mathcal{V})\, \|\, p_\theta(\cdot\mid X_S)\right)\right]
\end{aligned}
\end{equation}
where $C_j=\mathbb{E}[\log p^\star(X_{t+1}^j\mid \mathcal{V})]$ does not depend on $S$. Thus, maximizing \eqref{eq:stage1-pop} is equivalent (up to the constant $C_j$) to minimizing 
\begin{equation}
    \mathcal{L}_j(S) := \underbrace{\inf_{\theta:\,\mathrm{supp}(\theta)\subseteq S} \mathbb{E}\!\left[\mathrm{KL}\!\left(p^\star(\cdot\mid \mathcal{V})\,\|\,p_\theta(\cdot\mid X_S)\right)\right]}_{\eta_j(S)\ \geq 0}\;+\;\lambda |S|
\end{equation}

Let $B:=MB^{\mathrm{dyn}}(j)$. By definition, conditioning on $B$ is sufficient, thus the model is well-specified on $B$ and we have $\eta_j(B)=0$. For any strict superset $S\supsetneq B$, we can ignore extra variables in parameters, so $\eta_j(S)=0$ as well. But we have $|S| > |B|$, and therefore $\mathcal{L}_j(S) > \mathcal{L}_j(B)$.

Now consider any subset $S$ that does not contain $B$. By assumption of strict faithfulness and strict positivity, omitting any element of $B$ induces a strictly positive KL gap: $\eta_j(S) > 0$ whenever $B \nsubseteq S$. Therefore, there exists $\lambda_0$ such that for all $0 < \lambda \leq \lambda_0$, we have $\eta_j(S)$ dominating the sparsity advantage $\lambda(|B| - |S|)$ for any $S$ with $B \nsubseteq S$.

Therefore $B$ (or an equivalent minimal boundary when non-unique) is the minimizer of $\mathcal{L}_j(\cdot)$, and any minimizer corresponds to a maximizer of $\Psi_j(\cdot)$. 

\end{proof}

\subsection{Proof of \cref{prop:unrolled_acyclic}}
\label{app:unrolled_acyclic}

We consider the time unrolled directed graph corresponding to ~\eqref{eq:linear_svar}. 


Next we consider a time window $T\in \mathbb{N}$ and define the vertex set
\[
\mathcal{U}_{1:T} := \{(t,j): t\in\{1,\dots, T\}, \; j \in[d]\}
\]

We consider two types of directed edges:
(i) lagged edges across time slices, and, (ii) instantaneous edges within each time slice:
\begin{equation*}
\begin{aligned}
    (t - \ell, i) \to (t, j) &\quad \Longleftrightarrow \quad (A_\ell^\star)_{j i}\neq 0,\;\; \\
    & \ell\in \{1, \dots, L\},\;\; t-\ell\ge 1 \\
    (t,i) \to (t,j) &\quad \Longleftrightarrow \quad (B^\star)_{j i} \neq 0, \;\; i \neq j.
\end{aligned}
\end{equation*}
We denote the resulting directed graph by $\mathcal{G}_{1:T}$.

\begin{proof}
We observe that the directed graph generated by $B^\star$ is acyclic in nature, thus it permits a topological ordering $\pi:[d] \to \{1, \dots, d\}$ such that for every instantaneous edge in $B^\star_{ji} \neq 0$ i.e., $i \to j$, we have $\pi(i) < \pi(j)$.

Now lets define a strict lexicographic order $\prec$ on vertices $(t, j) \in \mathcal{U}_{1:T}$:

\[
\begin{aligned}
(t_1,j_1) \prec (t_2,j_2)
\Longleftrightarrow{}&
\ t_1 < t_2, \\
&\ \text{or } \bigl(t_1 = t_2
\text{ and } \pi(j_1) < \pi(j_2)\bigr).
\end{aligned}
\]

We observe that every directed edge in $\mathcal{G}_{1:T}$ points forward under the order $\prec$. In other words,
\begin{itemize}
    \item \emph{Lagged edges:} If $(t-\ell, i) \to (t, j)$ is a lagged edge, then $t - \ell < t$, thus $(t - \ell, i) \prec (t, j)$.
    \item  \emph{Instantaneous edges:} If $(t, i) \to (t, j)$ is an instantaneous edge, then we have $t = t$ and $\pi (i) < \pi (j)$ by definition, hence we have $(t, i) \prec (t, j)$.
\end{itemize}
Thus we observe that every edge points forward under the order $\prec$, and we can say $\mathcal{G}_{1:T}$ is acyclic in nature.

\end{proof}

\subsection{Assumptions for Identifiability}
\label{ass:ident_dyn}

We consider a restricted nonlinear dynamic SEM with additive noise. The observed process $\{X_t\}_{t\in\mathbb{Z}}$ satisfies
\begin{equation}
\label{eq:ident_dyn_sem}
\begin{aligned}
X_{t+1}^j &= g_j^\star\!\left(X_{t+1}^{\mathrm{pa}_0(j)},
X_t^{\mathrm{pa}_1(j)}, \ldots, X_{t+1-L}^{\mathrm{pa}_L(j)}
\right) + \varepsilon_{t+1}^j, \\
& j=1,\ldots,d,
\end{aligned}
\end{equation}
where:
\begin{enumerate}
    \item the instantaneous graph induced by $\{\mathrm{pa}_0(j)\}_{j=1}^d$ is acyclic

    \item the noise variables $\{\varepsilon_t^j : j \in [d],\, t \in \mathbb{Z}\}$ are jointly independent across $j$, i.i.d.\ over $t$, have positive densities, and are independent of $\{X_s\}_{s \le t}$

    \item the process is causally sufficient, meaning that there are no latent variables that act as common causes of two or more observed components of $X_t$; time-homogeneous, and has maximum lag order $L$

    \item the induced time-unrolled SEM belongs to an identifiable functional model class (IFMOC), for example a nonlinear additive-noise class \citep{peters2011ifmoc,peters2013causal}.

\end{enumerate}

Under the row-target, column-source convention used throughout the paper, the
lagged and instantaneous supports are defined by
\begin{equation}
\label{eq:ident_support_convention}
\begin{aligned}
(A_\ell^\star)_{ji}=1
\quad & \Longleftrightarrow \quad
i\in \mathrm{pa}_\ell(j),
\\
(B^\star)_{ji}=1
\quad &\Longleftrightarrow \quad
i\in \mathrm{pa}_0(j).
\end{aligned}
\end{equation}

\subsection{Proof of \cref{prop:ident_dyn}}
\label{app:proof_ident_dyn}
\begin{proof}
Consider a horizon $H \ge L+1$ and consider a finite collection of variables
\[ \mathcal{V}_H = \{(j,t) : j \in [d],\ t = 1,\ldots,H\}.\]
Next, we define the associated time unrolled graph
$G_H^\star = (\mathcal{V}_H,\mathcal{E}_H^\star)$ as follows: for any $t=L+1,\ldots,H$,
\begin{equation*}
\begin{aligned}
(i,t-\ell) \to (j,t) &\in \mathcal{E}_H^\star  \Longleftrightarrow \quad
(A_\ell^\star)_{ji}=1, \quad \ell=1,\ldots,L, \\
(i,t) \to (j,t) &\in \mathcal{E}_H^\star \Longleftrightarrow \quad
(B^\star)_{ji}=1.
\end{aligned}
\end{equation*}
The variables in the initial window
$\{(j,t): j\in[d],\, t=1,\ldots,L\}$ are treated as exogenous root nodes i.e. root nodes with no incoming edges or parents. We prove the claim in the following steps.

\paragraph{Step 1: $G_H^\star$ is a DAG.}
Since all lagged edges are strictly forward pointing in time, they cannot be part of a directed cycle. According to Assumption~\ref{ass:ident_dyn}(i), the instantaneous graph at each time slice is acyclic. Thus, for each time slice $t$, there exists a topological order $\pi_t$ on the vertices $\{(j,t): j\in[d]\}$ such that each instantaneous edge in the slice $t$ respects this order. Now we introduce a global order $\prec$ on $\mathcal{V}_H$ by
\[(i,s) \prec (j,t) \quad \Longleftrightarrow \quad
\bigl(s < t\bigr) \ \text{or}\ \bigl(s=t \text{ and } \pi_t(i) < \pi_t(j)\bigr). \]
We observe that every lagged edge respects $\prec$ because it goes from time $t-\ell$ to time $t$, and every instantaneous edge respects $\prec$ by construction of
$\pi_t$. Hence $G_H^\star$ contains no directed cycle and is a DAG.

\paragraph{Step 2: The unrolled process is a functional causal model on $G_H^\star$.}
For each $t=L+1,\ldots,H$ and $j\in[d]$, we can rewrite the structural equation
\eqref{eq:ident_dyn_sem} as:
\[X_t^j = g_j^\star\!\left( X_t^{\mathrm{pa}_0(j)},
X_{t-1}^{\mathrm{pa}_1(j)}, \ldots, X_{t-L}^{\mathrm{pa}_L(j)}
\right) + \varepsilon_t^j. \]
We can say that each node $(j,t)$ is a function of its parent variables in $G_H^\star$ along with an exogenous noise variable $\varepsilon_t^j$.
From Assumption~\ref{ass:ident_dyn}(ii), these additive noises are jointly independent across coordinates and independent of the past time steps. Therefore, the finite horizon distribution over $(X_1,\ldots,X_H)$ permits a structural equation model whose DAG is precisely $G_H^\star$.

\paragraph{Step 3: Identifiability of the time unrolled graph.}
Using Assumption~\ref{ass:ident_dyn}(iv), the induced unrolled SEM falls into an
identifiable functional model class. Therefore, by the identifiability results of
IFMOCs and their generalization to time series models with independent noise
\citep{peters2011ifmoc, peters2013causal}, the observational distribution of the
finite horizon process $(X_1,\ldots,X_H)$ uniquely identifies the DAG
$G_H^\star$.

In other words, there does not exist another DAG, $\widetilde G_H \neq G_H^\star$ along with structural mechanisms from the same model class that induce the exactly same observational distribution on $(X_1,\ldots,X_H)$.

\paragraph{Step 4: Recovery of $\{A_\ell^\star\}_{\ell=1}^L$ and $B^\star$.}
Given that the model is time-homogeneous by Assumption~\ref{ass:ident_dyn}(iii), the same lagged and instantaneous parent relationships repeat at every time slice. Therefore:
for each $\ell=1,\ldots,L$,
\begin{equation*}
\begin{aligned}
(A_\ell^\star)_{ji}=1 \quad \Longleftrightarrow \quad
(i,t-\ell)\to(j,t) \\
\text{ in } G_H^\star
\ \text{for any } t=L+1,\ldots,H.
\end{aligned}
\end{equation*}
and
\begin{equation*}
\begin{aligned}
(B^\star)_{ji}=1 \quad \Longleftrightarrow \quad
(i,t)\to(j,t) \\
\text{ in } G_H^\star
\ \text{for any } t=L+1,\ldots,H.
\end{aligned}
\end{equation*}
Since $G_H^\star$ is uniquely identified by the observational distribution, the
summary supports $\{A_\ell^\star\}_{\ell=1}^L$ and $B^\star$ are also uniquely
determined.

Finally, by way of contradiction, assume that there exists another parameter collection $\bigl(\widetilde B,\{\widetilde A_\ell\}_{\ell=1}^L\bigr)\neq
\bigl(B^\star,\{A_\ell^\star\}_{\ell=1}^L\bigr)$ within the same model class that
defines the same observational law of the process. This would mean that for every sufficiently large finite horizon $H$, it would generate the same distribution on
$(X_1,\ldots,X_H)$ but a different time-unrolled graph, thereby contradicting Step~3.
Therefore, we can say that $\{A_\ell^\star\}_{\ell=1}^L$ and $B^\star$ are identifiable within the assumed model class.

\end{proof}

\section{Complexity Analysis}
\label{app:complexity}

We evaluate the computational complexity of \textsc{SC3D} by segregating the costs of the Stage~1 (temporal preselection) and Stage~2 (constrained structure refinement). Let $d$ be the number of variables, $L$ be the maximum lag order, $T$ be the time window, and $N$ be the number of independent trajectories or samples. After constructing the temporal window, the total number of training samples is $n = N(T - L -1)$.

Each node-wise predictor is parameterized by a neural network with a hidden layer of width $H$. Let $E_1$ and $E_2$ denote the number of optimization epochs used in Stage~1 and Stage~2. The mini-batch size is represented by $B$.

In Stage~1, \textsc{SC3D}, fits $d$ independent node-wise predictive models. Each model takes the temporal predictor window as input
\[
    \mathcal{V}_t^{(j)} = \{X_{t+1}^{-j}, X_t, \dots, X_{t+1-L}\},
\]
which has dimensionality $dL + d_{\mathrm{inst}}$, where $d_{\mathrm{inst}} = d$ if instantaneous mode is enabled and $0$ otherwise. For each training step, the forward and backward passes scale linearly with the input dimension and hidden width. Therefore, the total time complexity of Stage~1 given to be
\begin{equation} \label{eq:stage1_comp}
    \mathcal{O}\!\left( E_1 \, n\, d\, H\, (dL + d_{\mathrm{inst}}) \right).
\end{equation}

Stage~2 executes on the masked input space generated by Stage~1. Let $\rho_{\mathrm{lag}} \in (0, 1]$ and $\rho_{\mathrm{inst}} \in (0, 1]$ represent the fraction of retained lagged and instantaneous edges after preselection. The total input dimension per node reduces to $ \rho_{\mathrm{lag}} dL + \rho_{\mathrm{inst}} d$.
Thus the total optimization cost for Stage~2 is 
\begin{equation} \label{eq:stage2_comp}
    \mathcal{O}\!\left( E_2 \, n\, d\, H\, (\rho_{\mathrm{lag}} dL + \rho_{\mathrm{inst}} d) \right).
\end{equation}

\paragraph{Acyclicity overhead}
\textsc{SC3D} deploys a spectral radius penalty to enforce the acyclicity on the instantaneous adjacency matrix $B$ via $K$ steps of power iteration. This computation which is performed periodically typically costs $\mathcal{O}\!\left(K d^2\right)$ per evaluation. If we consider all mini-batches and epochs, the total acyclicity overhead is
\begin{equation} \label{eq:acyc_comp}
    \mathcal{O}\!\left(E_2 \, (n / B_{\mathrm{batch}}) \, K d^2\right).
\end{equation}

Combining all components in~\eqref{eq:stage1_comp},~\eqref{eq:stage2_comp}, and~\eqref{eq:acyc_comp}, the total runtime complexity of \textsc{SC3D} can be expressed as
\begin{equation}
\begin{aligned}
    \mathcal{O}~\bigl( 
    & E_1 \, n\, d\, H\, (dL + d_{\mathrm{inst}}) +
    E_2 \, n\, d\, H\, (\rho_{\mathrm{lag}} dL + \rho_{\mathrm{inst}} d) \\ 
    & + E_2 \, (n / B_{\mathrm{batch}}) \, K d^2\bigr).
\end{aligned}
\end{equation}

\section{Datasets}
\label{app:datasets}
Inspired by the nonlinear temporal causal discovery scenarios studies in UnCLE~\cite{bi2025uncle}, we modify the NC8 and TVSEM benchmarks to generate controlled ground truth dynamics and extract specific elements of lagged and time varying causal structure to suit our experiments.
\subsection{Lorenz 96 Dataset}

We generate time-series data using the Lorenz 96 dynamical system \cite{lorenz1996predictability}, a standard model for chaotic behavior. The system consists of $d$ variables, $x_1, \dots, x_d$, governed by the differential equation:

\begin{equation}
    \frac{dx_i}{dt} = (x_{i+1} - x_{i-2})x_{i-1} - x_i + F,
\end{equation}

\noindent where indices are computed modulo $d$. The data is generated using an Euler discretization with additive Gaussian noise. The state update rule from time $t$ to $t+1$ is:

\begin{equation}
    \mathbf{x}_{t+1} = \mathbf{x}_t + \Delta t \cdot \frac{d\mathbf{x}_t}{dt} + \sigma \cdot \boldsymbol{\epsilon}_t,
\end{equation}

\noindent where $\boldsymbol{\epsilon}_t \sim \mathcal{N}(\mathbf{0}, \mathbf{I})$. The specific parameters used for the simulation are forcing constant $F = 8.0$, time step $\Delta t = 0.01$, noise scale $\sigma = 0.1$, and initialization: $\mathbf{x}_0 \sim \mathcal{N}(F \cdot \mathbf{1}, 0.01^2 \mathbf{I})$.

\subsection{Time-Varying Structural Equation Model (TVSEM)}

We simulate a non-stationary 2-dimensional system ($d=2$) using a regime-switching Vector Autoregressive model. The state evolution follows the equation:

\begin{equation}
    \mathbf{x}_t = \mathbf{A}^{(s_t)} \mathbf{x}_{t-1} + \boldsymbol{\epsilon}_t,
\end{equation}

\noindent where $\boldsymbol{\epsilon}_t \sim \mathcal{N}(\mathbf{0}, \sigma^2 \mathbf{I})$ is Gaussian noise with scale $\sigma=0.1$. The causal structure $s_t$ alternates every 200 time steps. In the first regime, the coupling coefficients are set to $0.8$ for $y \to x$ and $0.1$ for $x \to y$. In the second regime, these coefficients shift to $0.2$ and $0.7$ respectively, effectively reversing the dominant causal direction while maintaining the same underlying graph skeleton.

\subsection{NC8 Dataset}

We utilize an 8-dimensional system ($d=8$) with a lag order of $L=4$ that integrates linear autoregressive decay with heterogeneous non-linear interactions. The state update rule is defined as:

\begin{equation}
    \mathbf{x}_t = \text{clamp}\left( \sum_{\ell=1}^{L} \mathbf{A}^{(\ell)} \mathbf{x}_{t-\ell} + \mathbf{\Phi}(\mathbf{x}_{t-1}, t) + \boldsymbol{\epsilon}_t, -5, 5 \right)
\end{equation}

\noindent where $\boldsymbol{\epsilon}_t \sim \mathcal{N}(\mathbf{0}, \sigma^2 \mathbf{I})$ and the linear weights $\mathbf{A}^{(\ell)}$ decay by a factor of $(\ell+1)^{-1}$. The non-linear function $\mathbf{\Phi}$ applies distinct transformations to specific variables, including sinusoidal coupling, hyperbolic tangent activations, soft-cubic transformations $z \mapsto z^3(1+|z|)^{-1}$, and ReLU-like interactions $\max(\cdot, 0)$, while variable $x_4$ acts as a time-dependent exogenous driver. The data is simulated with a noise scale of $\sigma=0.1$ and bounded within $[-5, 5]$.

\subsection{Additional Results}
\subsubsection{Sensitivity to Stage-1 Masking Threshold}
\label{app:masking_ablation}

\begin{figure}
\label{fig:masking_sensitivity}
  \centering
  \includegraphics[width=0.85\columnwidth]{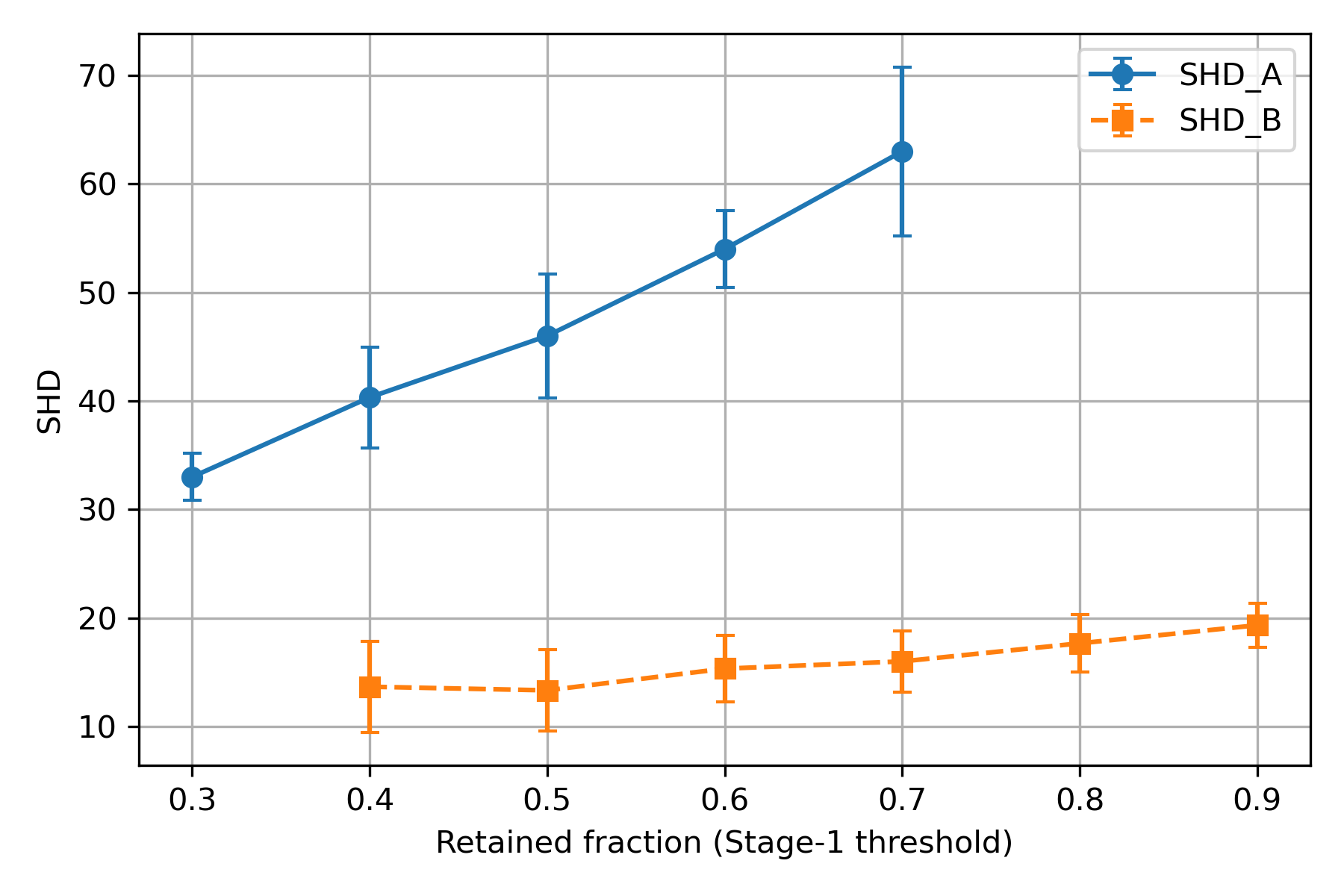}
  \caption{
  Sensitivity of \textsc{SC3D} to the Stage~1 masking threshold,
  measured by the retained fraction of candidate edges
  ($d=10$, $T=200$). Lower retained fractions correspond to
  stricter masks, while higher retained fractions allow more candidate
  edges into Stage~2. SHD$_A$ increases as the retained fraction grows,
  reflecting the larger lagged candidate set, whereas SHD$_B$ remains
  comparatively stable across thresholds.
  }
  \vspace{-0.5cm}
\end{figure}

The effect of \textsc{SC3D}'s sensitivity to the Stage~1 retention fraction is analyzed in \cref{fig:masking_sensitivity}. The stricter the masking procedure, the smaller the number of candidate edges used by the algorithm in Stage~2, which might help in suppressing the spurious lagged edges. However, this strategy could also risk on missing weaker true dependencies. For increasing values of the Stage~1 retention fraction, the lagged structure error increases slowly because there are more candidate edges for the subsequent Stage~2 filtering process to examine. In contrast, the instantaneous portion remains relatively stable across thresholds, implying that over selection can be mitigated through the constrained Stage~2 refinement and spectral acyclicity penalty.

\printcredits

\bibliographystyle{cas-model2-names}

\bibliography{references}





\end{document}